\documentclass[twoside]{article}
\usepackage{booktabs} 
\usepackage{amssymb,amsmath,amsthm, amsfonts}
\newtheorem{theorem}{Theorem}
\newtheorem{definition}{Definition}
\newtheorem{corollary}{Corollary}
\newtheorem{proposition}{Proposition}
\usepackage{wrapfig}
\usepackage{multirow}
\usepackage{subcaption}
\usepackage{caption}
\usepackage{graphicx}
\usepackage{hyperref}
\usepackage{enumitem}
\usepackage{natbib}
\usepackage[normalem]{ulem}
\useunder{\uline}{\ul}{}
\usepackage{algorithm}
\usepackage{algpseudocode} 
\renewcommand{\algorithmicrequire}{\textbf{Input:}}
\renewcommand{\algorithmicensure}{\textbf{Output:}}

\algrenewcommand{\algorithmicrequire}{\textbf{Input:}}
\algrenewcommand{\algorithmicensure}{\textbf{Output:}}

\usepackage[accepted]{aistats2025}

\begin{document}

\runningtitle{Incremental Uncertainty-aware Performance Monitoring with Active Labeling Intervention}

\runningauthor{Alexander Koebler, Thomas Decker, Ingo Thon, Volker Tresp, Florian Buettner}

\twocolumn[

\aistatstitle{Incremental Uncertainty-aware Performance Monitoring \\ with Active Labeling Intervention}

\aistatsauthor{Alexander Koebler$^{*1,2}$ \quad Thomas Decker$^{*1,3,4}$ \quad Ingo Thon$^1$ \quad
Volker Tresp$^{3,4}$ \quad Florian Buettner$^{2,5,6}$ }

\aistatsaddress{ $^1$Siemens AG \quad $^2$Goethe University Frankfurt \quad $^3$LMU Munich \quad 
$^4$Munich Center for Machine Learning (MCML) \\\quad $^5$German Cancer Research Center (DKFZ)
 \quad $^6$German Cancer Consortium (DKTK)} ]
\begin{abstract}
We study the problem of monitoring machine learning models under gradual distribution shifts, where circumstances change slowly over time, often leading to unnoticed yet significant declines in accuracy. To address this, we propose Incremental Uncertainty-aware Performance Monitoring (IUPM), a novel label-free method that estimates performance changes by modeling gradual shifts using optimal transport. In addition, IUPM quantifies the uncertainty in the performance prediction and introduces an active labeling procedure to restore a reliable estimate under a limited labeling budget. Our experiments show that IUPM outperforms existing performance estimation baselines in various gradual shift scenarios and that its uncertainty awareness guides label acquisition more effectively compared to other strategies. 
 
\end{abstract}

\section{INTRODUCTION}
\begin{figure}[htb!]
	\centering
 \includegraphics[width=0.48\textwidth]{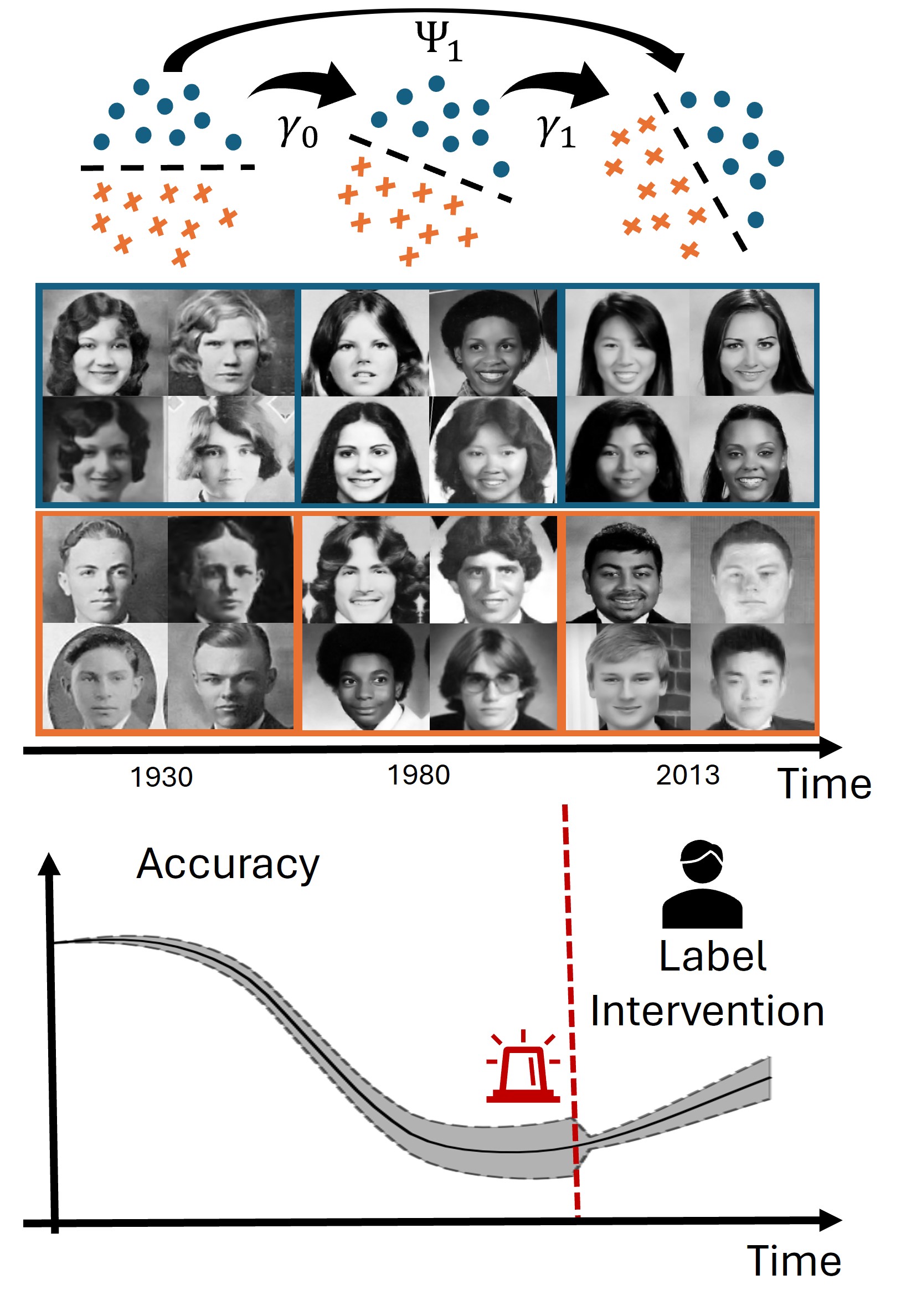}
	\caption{
 Illustration of Incremental Uncertainty-aware Performance Monitoring (IUPM) to estimate performance changes over time using only labels from the initial training distribution. By iteratively linking unlabeled data points  using optimal transport couplings $\gamma$ and combining them into an overall transition map $\Psi$, it can anticipate the true model performance under gradual shifts. IUPM also provides an inherent uncertainty measure and an active labeling procedure to efficiently reduce uncertainty and improve estimation reliability under a limited labeling budget.}
 
	\label{fig:fig1}
\end{figure}
Deployed machine learning models often face the critical challenge of distribution shifts, where the data encountered in production deviates from the data used during training. Many relevant shift scenarios involve changes over time, which are often gradual and continuous \citep{yao2022wild, xie2024evolving}. These shifts are characterized by the fact that the statistical properties of the data or the environment change progressively rather than abruptly. This property can make gradual shifts more insidious, as they may not be immediately apparent but can still lead to substantial degradation in prediction quality over time \citep{gama2014survey}. Therefore, anticipating and understanding temporal performance changes is essential for ensuring the reliability and effectiveness of a machine-learning model in dynamic environments. However, directly monitoring the performance during deployment is challenging as labeled data is often unavailable in production. Moreover, obtaining labels can be cumbersome, time-consuming, and costly, leading to delays in the assessment. Therefore, an increasing number of label-free estimation methods have been proposed that aim to anticipate the model performance purely based on unlabeled data available at runtime \citep{yu2024survey}. Such methods leverage diverse strategies to approach this task by for instance estimating performance changes based on feature statistics \citep{deng2021labels}, the agreement of multiple models \citep{jiang2022assessing, baek2022agreementontheline} or by analyzing the model's confidence \citep{guillory2021predicting,gargleveraging}. However, none is explicitly tailored to the particular nature of gradual shifts, which are ubiquitous in practice. On top of that, existing techniques suffer from two fundamental limitations. First, they cannot quantify any uncertainty related to the performance estimate, causing ambiguity about when to trust the estimate and when to rather collect extra labels to obtain a more accurate assessment. Second, current methods do not guide how to best support performance estimation in monitoring situations where a limited number of labels could be acquired. This neglects the fact that actively selecting specific data points of interest to be labeled can greatly improve the effectiveness of model evaluations under a limited labeling budget \citep{kossen2021active}.
To address these limitations we make the following contributions:
\begin{itemize}[leftmargin=16pt]
\item We propose \textbf{Incremental Uncertainty-aware Performance Monitoring (IUPM)}, as a novel label-free performance estimation method tailored to gradual distribution shifts with an inherent notion of uncertainty.
\item We introduce an active intervention step to increase the reliability of the performance estimate by labeling examples contributing the highest uncertainty to the performance estimate.
 \item We show that IUPM works provably well for a general class of gradual shifts and demonstrate its benefits over existing baselines on different datasets under synthetic and real gradual shifts.
\end{itemize}

\section{BACKGROUND AND RELATED WORK}
\paragraph{Label-free performance estimation}
Label-free performance estimation methods aim to anticipate the predictive quality of a machine learning model in out-of-distribution settings purely based on unlabeled data \citep{yu2024survey}. Many existing approaches with theoretical guarantees are typically restricted to certain shifts, such as considering a pure covariate or label shift \citep{sugiyama2008direct, chen2021mandoline,garg2020unified}, or require additional properties of the model's output confidence \citep{guillory2021predicting, lu2023characterizing}. 
On the other hand, a growing number of methods have been introduced that demonstrate promising empirical results for specific model classes and shift types. Such techniques leverage diverse strategies based on feature statistics \citep{deng2021labels}, the model's behavior under test-time augmentations \citep{deng2021does}, thresholding the model's confidence \citep{gargleveraging}, the agreement between different models \citep{jiang2022assessing, baek2022agreementontheline}, or model differences after retraining \citep{yu2022predicting}. However, none of them is explicitly tailored to gradual distribution shifts and their inherent structure or comes with an applicable notion of uncertainty, which is the scope of our work. 
\paragraph{Gradual Domain adaptation}
While the special characteristics of gradual shifts have not been considered for performance estimation, they have already been explored in the context of domain adaptation. In general, domain adaptation addresses the problem of increasing the performance of a model trained on a labeled source domain when applied to a novel target domain \citep{wilson2020survey}. Classical methods typically act in a one-shot fashion, where adaption is performed directly between the two domains. However, this approach can struggle in situations where the encountered distribution shift is particularly strong \citep{he2023gradual}. Alternatively, gradual domain adaptation \citep{kumar2020understanding, wang2020continuously, wang2022understanding} works sequentially by introducing a series of intermediate domains. These domain interpolations serve as stepping stones to reduce the complexity of the overall distribution shift, which can help to increase the adaptation effectiveness \citep{abnar2021gradual, chen2021gradual}. When monitoring a model during deployment, however, it is unclear at which point adaptation is truly necessary without having a faithful estimate of the current model performance at runtime. This additionally motivates the need for performance estimation methods that explicitly consider the gradual nature of shifts over time. 

\paragraph{Active risk estimation}
In contrast to label-free performance estimation, the field of active risk estimation \citep{sawade2010active} is concerned with reducing the number of labels required to yield a reliable estimate of the model's performance under an existing but limited labeling budget. A variety of methods \citep{sawade2010active, kossen2021active, kossen2022active, Lee2024TowardsOM} have been developed to approach this task analog to active learning by introducing different sampling strategies to query and label test samples. Most prominently, the authors in \citep{sawade2010active} introduce an importance sampling approach whereas the authors in \citep{kossen2021active, kossen2022active} utilize a surrogate model using Gaussian process models or Bayesian neural networks to estimate which samples would contribute most to the performance estimation or even use the surrogate to predict the loss of the target model directly \citep{kossen2022active}.

\paragraph{Optimal Transport} Optimal Transport (OT) aims at finding the cost-minimizing way to transform one probability measure into another \citep{peyre2019computational}. Consider having $n_0$ samples from a domain $\Omega_0 = \{x_0^i\}_{i=1}^{n_0}$ and $n_1$ samples from another domain $\Omega_1 = \{x_1^i\}_{i=1}^{n_1}$ with corresponding empirical distributions
\begin{align*}
	&\hat{p}_0 = \sum_{x_0 \in \Omega_0} \frac{1}{n_0} \delta_{x_0}
	&\hat{p}_1  = \sum_{x_1 \in \Omega_1} \frac{1}{n_1} \delta_{x_1}
\end{align*} where $\delta_{x}$ denotes the Dirac measure. For a cost function $c: \Omega_0 \times \Omega_1 \rightarrow \mathbb{R}^+$, the transformation of $\hat{p}_0$ into $\hat{p}_1$ can be formalized by a coupling $\gamma$ which represents a valid distributions over $(\Omega_0 \times \Omega_1)$ with marginals corresponding to $\hat{p}_0$ and $\hat{p}_1$. Identifying the cost-optimal coupling reads:
\begin{align*}
	\hat{\gamma} = \arg \min_{\gamma \in \Gamma} \sum_{x_0 \in \Omega_0 }\sum_{x_1 \in \Omega_1 } c(x_0, x_1) \gamma(x_0,x_1) \quad \text{with} \\ \quad \Gamma =\{\gamma \in \mathbb{R}^{n_0 \times n_1}\; |\; \gamma \mathbf{1}_{n_1} = \hat{p}_0,  \gamma^T\mathbf{1}_{n_0} = \hat{p}_1\} 
\end{align*} which can be solved using different algorithmic approaches \citep{peyre2019computational}. In the discrete sample case the obtained $\gamma \in \mathbb{R}^{n_0 \times n_1}$ simply is a matrix with entries $\gamma(x_0, x_1)$. Moreover, the conditional coupling 
\begin{align*}
	\gamma(X_0=x_0|X_1=x_1) = \frac{\gamma(x_0, x_1)}{\sum_{x_0 \in \Omega_0} \gamma(x_0, x_1)}
\end{align*} is a left-stochastic matrix whose entries can be interpreted as transition probabilities when moving from samples of $X_1$ to samples of $X_0$ following the most cost-efficient path. OT couplings can further be used to characterize how similar two distributions are using the corresponding Wasserstein distance:
\begin{align*}
    \mathcal{W}_p(P_0, P_1) = \inf_{\gamma \in \Gamma} \mathbb{E}_{(x_0, x_1) \sim \gamma}[c(x_0, x_1)^p]^{1/p}
\end{align*}
Intuitively, the Wasserstein distance describes the expected transportation cost under the optimal coupling. In the following, we will only consider the case where $p=1$ such that we denote $\mathcal{W} := \mathcal{W}_1$.
\section{INCREMENTAL UNCERTAINTY-AWARE PERFORMANCE MONITORING}
\paragraph{Problem Setup}
 Consider a machine learning model $f:\mathcal{X}\rightarrow \mathcal{Y}$ that has been trained with labeled data $(X_0, Y_0)$ from the distribution $P_0(X_0, Y_0)$ over the space $(\mathcal{X} \times \mathcal{Y})$. Suppose $f$ is deployed in order to make predictions over time $t>0$ with respect to data $\{(X_t, Y_t)\}_{t=1}^T$ each distributed with $P_t(X_t, Y_t)$. In this work, we are mainly interested in gradual shifts, which are typically characterized as follows \citep{kumar2020understanding, he2023gradual}:
 \begin{definition} A distribution shift over $\{(X_t, Y_t)\}_{t=0}^T$ is gradual in time $t=0, \dots, T$ if the Wasserstein distance between two consecutive steps is bounded:
 \begin{align*}
     \mathcal{W}(P_t(X_t, Y_t), P_{t-1}(X_{t-1}, Y_{t-1})) \le \Delta_t \quad \forall t=1, \dots, T
 \end{align*}
 \end{definition}
In our setup, we consider a slightly more refined property to characterize gradual shifts:
\begin{definition} A distribution shift over $\{(X_t, Y_t)\}_{t=0}^T$ is called gradually Lipschitz smooth in $X_t$ if there is a cost function $c: \mathcal{X}\times \mathcal{X}\rightarrow \mathbb{R}^+$ for which we have $\mathcal{W}(P_t(X_t), P_{t-1}(X_{t-1})) \le \varepsilon_t $ for all $t=1, \dots , T$ and there exist constants $L_t>0$ such that:
\begin{align*}
    \mathcal{W}\big(P_t(Y_t|X_t), P_{t-1}(Y_{t-1}|X_{t-1})\big)\le L_t \; c(X_t, X_{t-1})
\end{align*}
\end{definition}
Intuitively this property implies that if two data points in consecutive domains are close in terms of a cost function $c$, their conditional target distributions should also be similar. Note that this assumption is quite reasonable for gradual shifts in general and one can easily show that any shift that is gradually Lipschitz smooth in $X_t$ is also gradual as defined above:
\begin{proposition}
If a distribution shift over $\{(X_t, Y_t)\}_{t=0}^T$ is gradually Lipschitz smooth in $X_t$ with constants $L_t$, then it is also gradual:
\begin{align*}
    \mathcal{W}(P_t(Y_t, X_t), P_{t-1}(Y_{t-1}, X_{t-1})) \le (1+L_t)\;\varepsilon_t := \Delta_t
\end{align*}
\end{proposition}

\paragraph{Incremental Performance Estimation}
Given the setup introduced above, remember that during runtime $(t>0)$ one typically only has access to unlabeled data from $X_t$ and our goal is to estimate how well a model $f$ performs over time solely based on this information. 
To do so effectively in the context of gradual shifts we propose an incremental approach: Let $\gamma_t(X_{t-1}|X_t)$ be the conditional coupling linking data from $X_t$ to data from $X_{t-1}$ in a cost-efficient way. Further, we define $\Psi_t(X_0|X_t):=\prod_{i=1}^t \gamma_i(X_{i-1}|X_i)$ describing the transition matrix obtained from composing all incremental transition matrices $\gamma_i(X_{i-1}|X_i)$ via matrix multiplication. It expresses the overall transition probabilities of going back to the labeled data available at $t=0$ by connecting samples of two subsequent time points incrementally using an individual optimal transport coupling. Based on this we propose the following strategy to estimate missing labels for performance evaluation over time: 
\begin{align*}
	\hat{P}(Y_t|X_t)=\mathbb{E}_{\Psi_t(X_0|X_t)}\left[P(Y_0|X_0) \right] 
\end{align*} 
This means that our label estimate $\hat{P}(Y_t|X_t)$ arises as mixture distribution \citep{everitt2013finite} combining labeled data in $X_0$ according to the accumulated incremental coupling results. This strategy is explicitly motivated by temporal shifts as we leverage their gradual nature by modeling subsequent distribution shifts incrementally using Optimal Transport. 

Given a loss function $\mathcal{L}$ to measure model performance and a set of samples $\Omega_t$ from $X_t$, the resulting performance estimate for IUPM  at time $t$, denoted by ($\hat{\mathcal{L}}_t^{\textit{IUPM}})$, is given by:
\begin{align*}
	\hat{\mathcal{L}}_t^{\textit{IUPM}} &= \mathbb{E}_{P(X_t)}\mathbb{E}_{\hat{P}(Y_t|X_t)} \left[ \mathcal{L}(f(X_t), Y_t) \right] \\ &= \frac{1}{n_t} \sum_{x_t\in \Omega_t} \mathbb{E}_{\hat{P}(Y_t|X_t=x_t)} \left[ \mathcal{L}(f(X_t), Y_t) \right]
\end{align*}

The following theorem shows that this estimate will be close to the true model performance if the encountered shift is gradually Lipschitz smooth: 

\begin{theorem}
    Let $\mathcal{L}: \mathcal{Y}\times \mathcal{Y} \rightarrow \mathbb{R}$ be a loss function that is 1-Lipschitz in its second argument and denote the true model performance at time $t$ with $\mathcal{L}_t$.
    If a distribution shift over $\{(X_t, Y_t)\}_{t=1}^T$ is gradually Lipschitz smooth in $X_t$, then:
    \begin{align*}
        | \mathcal{L}_t - \hat{\mathcal{L}}_t^{\textit{IUPM}}| \le   \sum_{i=1}^t L_t \varepsilon_t
    \end{align*}
\end{theorem}
Note that the worst-case performance estimation error grows only linearly in $t$ matching corresponding bounds on the adaptation error established in the gradual domain adaptation literature \citep{he2023gradual}.
This theorem also implies that relying explicitly on the incremental OT couplings to relate unlabeled samples to labeled ones is most effective:
\begin{corollary}
If each $\gamma_t$ is the optimal transport coupling between two consecutive domains $X_{t-1}$ and $X_t$, then this minimizes the estimation error across all possible incremental couplings. 
\end{corollary}
As a consequence, IUPM provides a theoretically grounded approach to estimating the performance of machine learning models during deployment facing gradual distribution shifts. The underlying proofs and further mathematical details for all theoretical statements made above are provided in \hyperref[app:A]{Appendix A}.

\paragraph{Uncertainty Estimation and Active Labeling Intervention}\label{sec:uncertainty_estimation}
Another implication of the Theorem above is that the estimation error might grow linearly in time. Therefore it would also be desirable to additionally quantify the uncertainty related to the assessment. In cases of low confidence estimates it is preferable to rather collect true labels under the current data distribution to verify the performance indication. Our approach also provides a simple way to achieve this as IUPM exhibits an intrinsic notion of uncertainty due to the incremental matching procedure. More specifically, the estimate $\hat{P}(Y_t|X_t)$ is an actual predictive distribution that also internalizes uncertainty for cases where linked samples have contradicting labels. Therefore, we can use it to quantify the uncertainty of the anticipated performance using the expected standard deviation (SD) of the sample-wise loss estimates:
\begin{align*}
	\mathcal{U}(\hat{\mathcal{L}}_t^{\textit{IUPM}}) = \mathbb{E}_{P(X_t)} \text{SD}_{\hat{P}(Y_t|X_t)}\left[\mathcal{L}(f(X_t), Y_t)\right]
\end{align*}
In addition to providing a means for users to consolidate their trust in the performance estimate, the quantified uncertainty
$\mathcal{U}(\hat{\mathcal{L}}_t^{\textit{IUPM}})$
can also be used to automatically trigger efficient relabeling when the uncertainty exceeds a significant threshold. We utilize this property by proposing a novel sampling strategy for active label intervention. With our Uncertainty Intervention (UI) strategy we aim to make the most efficient use of a limited labeling budget allowing us to label only $m$ samples. For this, we query ground truth labels for critical samples $x_t$ that contribute the largest to the uncertainty of the anticipated performance for the current step:
\begin{align*}
	\arg\text{top-}m_{x_t \in \Omega_t} \mathcal{U}\left[\mathcal{L}(f(x_t), Y_t)\right]
\end{align*} where $\arg\text{top-}m$ denotes the operator selecting the top $m$ elements maximizing the objective. 
The new labels are used to update $\Psi_t(X_0|X_t)$ such that $\hat{P}(Y_t|X_t=x_t)$ assigns a fixed label removing the accumulated uncertainty for sample $x_t$.
\section{EXPERIMENTS}
In this section, we evaluate our IUPM approach on three different gradual shift scenarios. First, to assess the general functionality of IUPM and to yield insights into our proposed uncertainty intervention sampling strategy, we present results based on synthetic examples with continuous shifts in two-dimensional space. Second, we analyze the capabilities of IUPM to monitor performance changes due to different gradual image perturbations on MNIST \citep{mnist} and a subset of ImageNet \citep{Howard_Imagenette_2019}. Finally, we apply IUPM to a real-world shift scenario based on yearbook portraits across several decades \citep{ginosar2015century} and demonstrate its superiority over several baselines. For this dataset, we also provide guidance on how to validate the underlying theoretical assumptions empirically.
The code for IUPM is made available\footnote{\texttt{https://github.com/alexanderkoebler/IUPM}}, the algorithm is detailed in \hyperref[app:B]{Appendix B} and the setup for each experiment is documented in \hyperref[app:C]{Appendix C}. 
\subsection{Evaluation Setup}
\paragraph{Performance Estimation Baselines}
Throughout the experiments, we compare our approach to several existing performance estimation methods. Those baselines consist of four confidence-based error estimation methods, as described in \citep{gargleveraging}. \textit{Average Confidence (AC)} simply estimates the prediction accuracy as the expectation of the confidence for the predicted class across the data set in step $t$ as:
\begin{equation*}
AC_{\Omega_t} = \mathbb{E}_{x\sim \Omega_t}[\max_{j\in \mathcal{Y}} f_j(x)]
\end{equation*}
The more sophisticated \textit{Difference Of Confidence (DOC)} \citep{guillory2021predicting} uses the discrepancy between the model confidence on the source and target data sets as an estimate of performance degradation. To obtain an approximation of the performance in step $t$, the degradation is subtracted from the performance on the initialization data set $t=0$ according to:
\begin{align*}
DOC_{\Omega_t} &= \mathbb{E}_{x, y\sim \Omega_0}[\arg\max_{j\in \mathcal{Y}} f_j(x) \neq y]\\ &+ \mathbb{E}_{x\sim \Omega_t}[\max_{j\in \mathcal{Y}} f_j(x)] - \mathbb{E}_{x\sim \Omega_0}[\max_{j\in \mathcal{Y}} f_j(x)]
\end{align*}
\textit{Average Threshold Confidence (ATC)} \citep{gargleveraging} learns a threshold for model confidence on the initialization data set $\Omega_0$ and estimates accuracy on the current set as the fraction of examples where model confidence exceeds this threshold. Lastly, we consider \textit{Importance re-weighting (IM)}, as proposed by \citep{chen2021mandoline}. We also evaluate the direct mapping $\gamma(X_0|X_t)$ for label transport and performance estimation inspired by \citep{decker2024explanatory} and call it \textit{Non-Incremental Performance Estimation (NIPM)}. Across all experiments, we consider the model accuracy as the loss criterion $\mathcal{L}$ to evaluate performance.

\paragraph{Sampling Strategies}
Throughout all active label intervention experiments, we utilize a fixed threshold for our uncertainty indicator of $\mathcal{U}(\mathcal{L}_t) > 0.1$ to trigger an intervention step. We provide an ablation study detailing the rationale for choosing this threshold value in \hyperref[app:D]{Appendix D}.
To evaluate the efficacy of our introduced \textit{Uncertainty Intervention (UI)} sampling strategy introduced in Section \ref{sec:uncertainty_estimation}, we compare two baseline methods.
First, we draw from the active testing literature and utilize an active sampling strategy introduced by \citep{kossen2021active}, which selects samples based on their expected loss contribution under the performance estimate measured based on the cross-entropy loss:
\begin{align*}
	\arg\text{top-}m_{x_t \in \Omega_t} -\sum_y \hat{P}(y_t|X_t=x_t)\log f(x_t)_y
\end{align*}
We refer to this sampling strategy as \textit{Cross Entropy Intervention (CEI)}.
Note that compared to \citep{kossen2021active}, we only trigger the labeling procedure once our uncertainty threshold is exceeded and not for every performance calculation directly using the $m$ generated labels.
Lastly, we introduce \textit{Random Intervention (RI)} as a naïve but effective baseline, where we randomly sample $m$ samples from the available set $\Omega_t$ at time step $t$.

\subsection{Experimental Results}
\paragraph{Translation and Rotation in Input Space}
In our first experiment, we use three two-dimensional toy data sets (Figure \ref{fig:syn_shifts}) provided by \citep{scikit-learn}. For all three datasets, we train a Random Forest (RF), XGBoost (XGB), and a Multilayer Perceptron (MLP) classifier in the initial source distribution. 
\begin{figure}[htb!]
	\centering
 \includegraphics[width=0.42\textwidth]{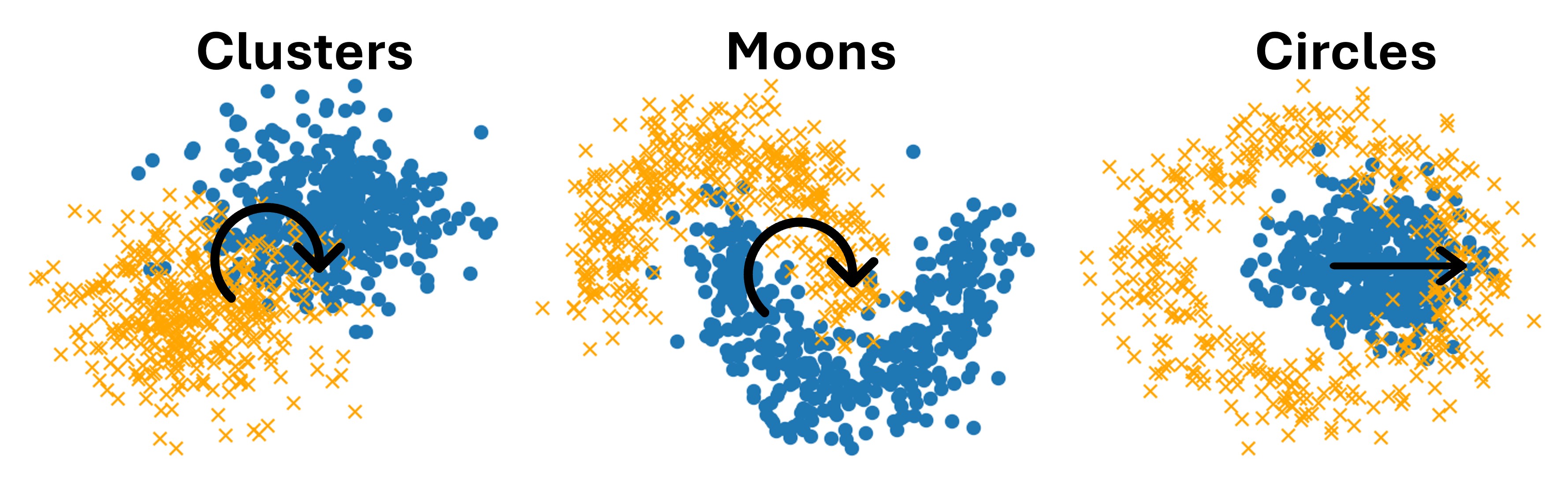}
	\caption{Synthetic two-dimensional toy datasets and corresponding shifts indicated by black arrows.}
	\label{fig:syn_shifts}
\end{figure}
After training at $t=0$, all data sets are shifted for $t=1, \dots, 100$ steps to simulate gradual changes over time. For the "Clusters" and "Moons" data sets, this shift results from rotating both classes by 2° per step. The "Circles" dataset experiences a translation shift by $0.02$ in the x-direction only on the inner circle class.\\
By showing the performance estimation over time for monitoring an MLP model on the moons data set facing a gradual rotation, we exemplarily illustrate in Figure \ref{fig:moons_performance} that IUPM best estimates the actual performance degradation.
\begin{figure}[htb!]
	\centering
 \includegraphics[width=0.45\textwidth]{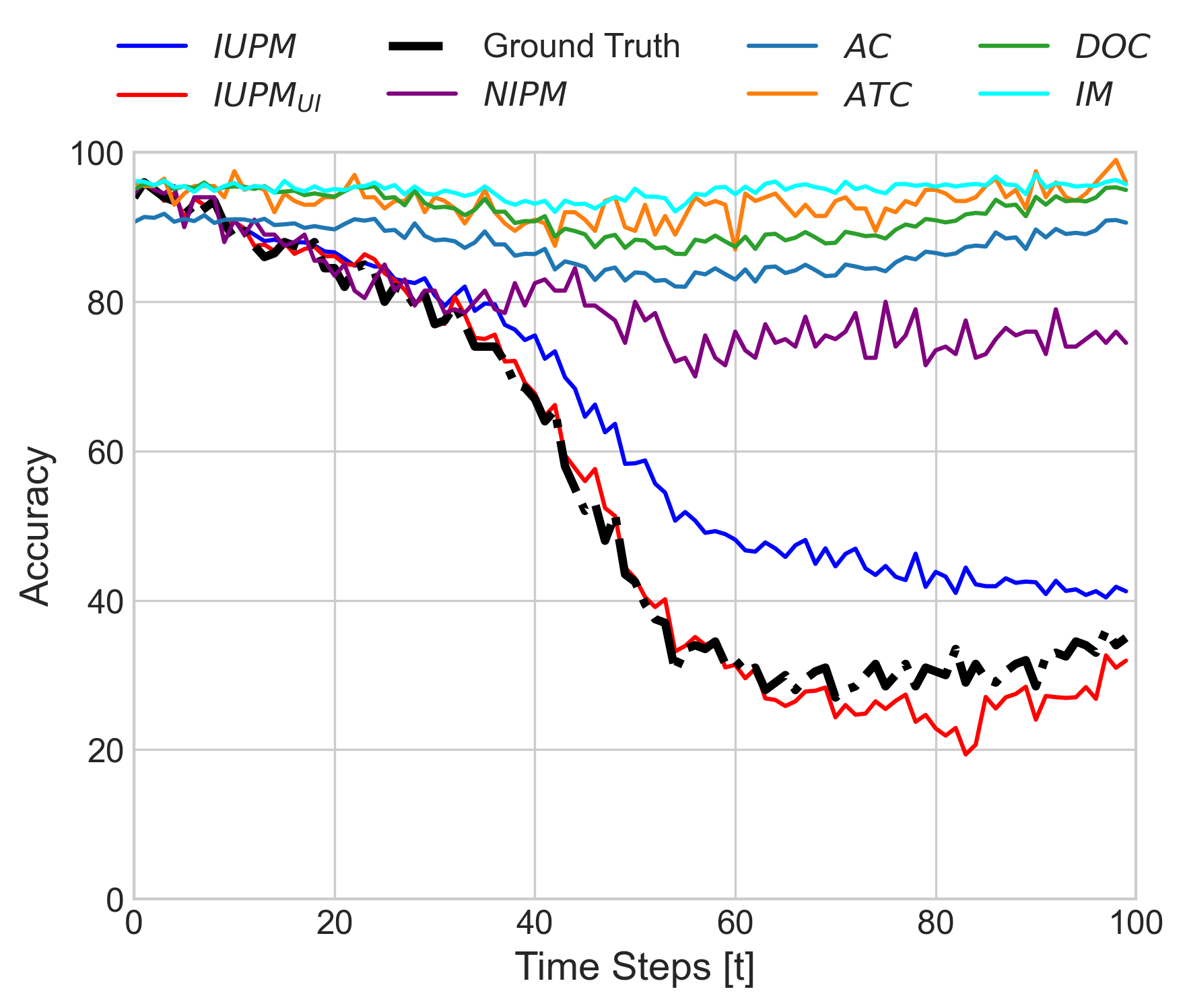}
	\caption{Performance estimation for an MLP model on the synthetic moons data set for a rotational shift over 100 steps resulting in a total rotation of 200°. Our proposed IUPM approach with and without label intervention clearly yields the highest fidelity for the performance estimation.}
	\label{fig:moons_performance}
	\vspace{-0.4cm}
\end{figure}
While initially, the non-incremental NIPM approach can also approximate the real performance, it deviates strongly as soon as the degree of degradation caused by the shift accelerates. While in early steps, it is still possible to directly match source samples $x_0$ to target samples $x_t$ sharing the same class, this becomes increasingly difficult as with the rotation, the conditional distributions $P_0(Y_0|X_0)$ and $P_t(Y_t|X_t)$ diverge. The confidence-based approaches fail to sufficiently pick up the gradual performance degradation. These observations are consistent across three different models and data sets (Table \ref{tab:syn_shifts}). As shown in the previous figure and table, the error of the accuracy estimate by IUPM can further be significantly reduced by introducing active label intervention steps. For this, we compare our proposed UI approach with CEI and RI. In all cases, we relabel $50\%$ of $\Omega_t$ when the total uncertainty exceeds a predefined
threshold.
As described above, this threshold is set to 10\% in accuracy deviation. Due to its empirical success, we used the same threshold across all subsequent experiments.
As shown in Figure \ref{fig:moons_intervention}, whilst all sampling methods can keep the uncertainty below the predefined threshold, our proposed UI approach requires far fewer intervention steps and, thus, labeling effort. This is underlined by Table \ref{tab:sampling_methods} showing that UI, in most cases, outperforms the other approaches concerning the error in the performance estimation while consistently requiring far fewer intervention steps. The superiority of UI in terms of intervention efficiency is further confirmed by an ablation study in \hyperref[app:D]{Appendix D}.
\begin{figure}[htb!]
	\centering
	\includegraphics[width=0.48\textwidth]{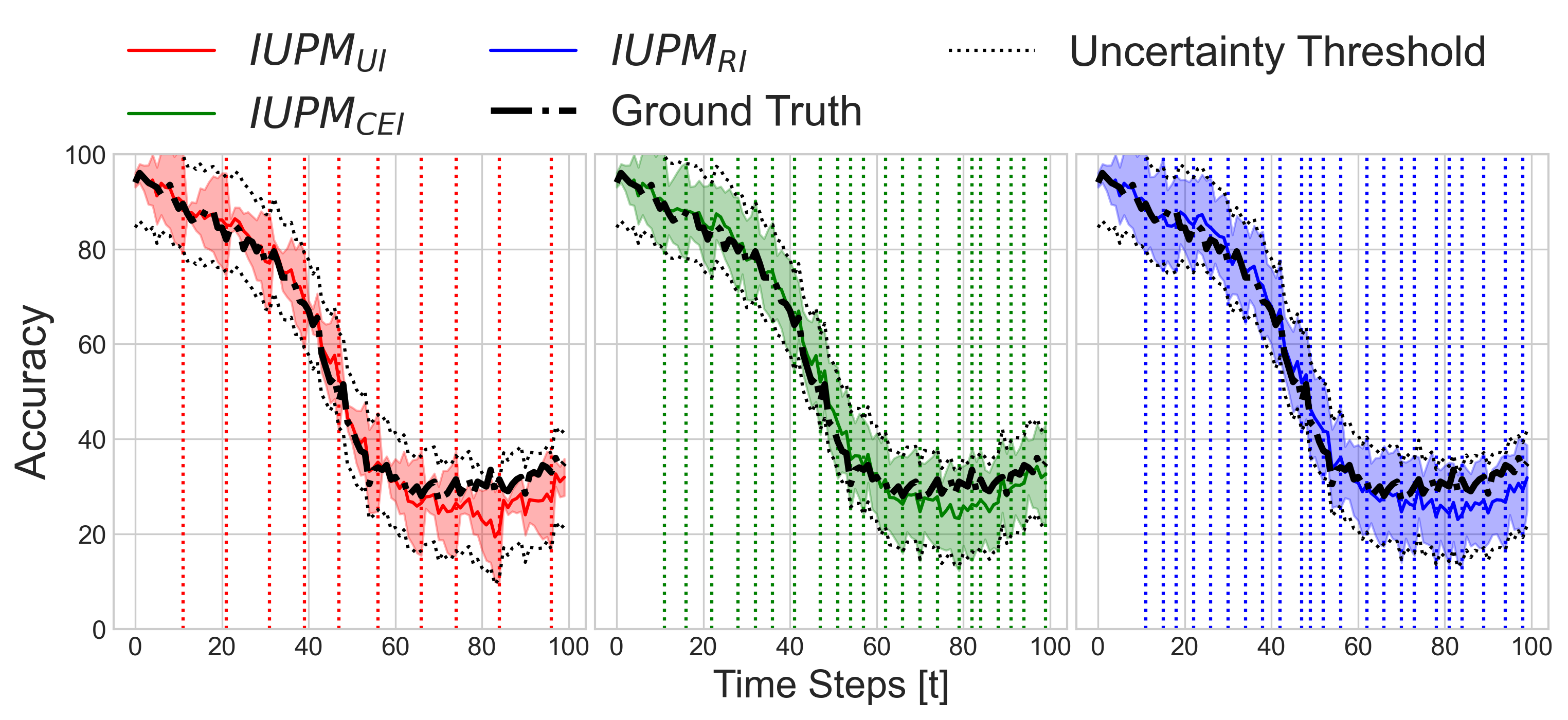}
	\caption{Comparison of sampling strategy using Active Label Intervention on moons data set over 100 steps. All intervention strategies allow keeping the uncertainty below the predefined threshold, however, our proposed Uncertainty Intervention (UI) requires far fewer intervention steps.}
	\label{fig:moons_intervention}
	\vspace{-0.2cm}
\end{figure}
\begin{figure}[htb!]
	\centering
	\includegraphics[width=0.48\textwidth]{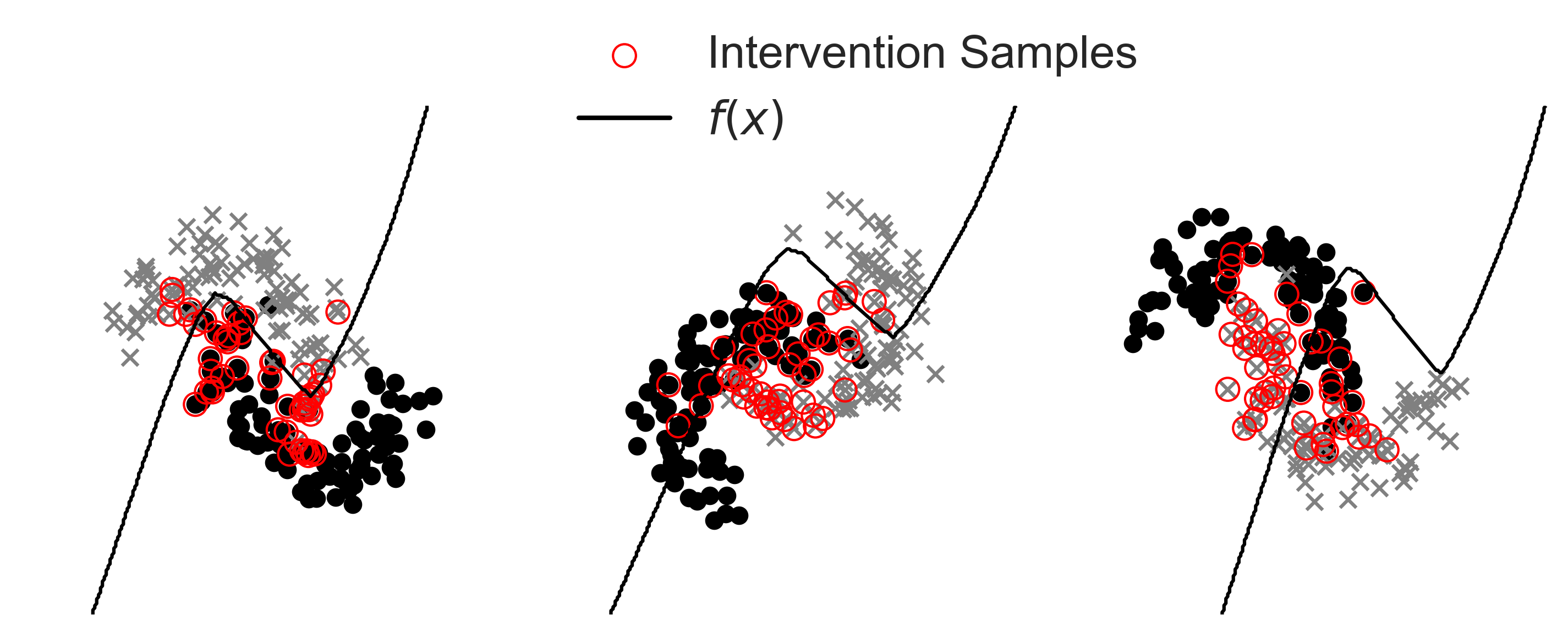}
	\caption{Visualization of 50 samples drawn using our proposed Uncertainty Intervention (UI) sampling strategy before, in the middle, and at the end of 100 time steps. The method consistently samples data points along the decision boundary inherently defined by the moons data set.}
	\label{fig:moons_sampling}
    \vspace{-0.5cm}
\end{figure}
Figure \ref{fig:moons_sampling} provides another intuitive illustration of the strength of the proposed sampling strategy. UI selects the most relevant examples that are close to the inherent decision boundary of the dataset.
\begin{table*}[htb!]
\caption{Mean Average Error (MAE) between ground truth and estimated accuracy using baseline methods and IUPM across three synthetic data sets and three different models. The table shows the mean across five random seeds, we refer to \hyperref[app:D]{Appendix D} for confidence intervals.}
\centering
\label{tab:syn_shifts}
\begin{tabular}{@{}cccccccccc@{}}
\toprule
\multirow{2}{*}{\textbf{Method}} &
\multicolumn{3}{c}{\textbf{Clusters}} &
\multicolumn{3}{c}{\textbf{Moons}} &
\multicolumn{3}{c}{\textbf{Circles}} \\ \cmidrule(l){2-10} 
& \textbf{RF} & \textbf{XGB} & \textbf{MLP} & \textbf{RF} & \textbf{XGB} & \textbf{MLP} & \textbf{RF} & \textbf{XGB} & \textbf{MLP} \\ \midrule
\textbf{ATC}  & 0.4413       & 0.4664        & 0.4348        & 0.3788       & 0.3594        & 0.3679        & 0.3514       & 0.3442        & 0.3531        \\ \midrule
\textbf{AC}   & 0.4313       & 0.4875        & 0.4402        & 0.3243       & 0.3741        & 0.3529        & 0.2760       & 0.3325        & 0.3240        \\ \midrule
\textbf{DOC}  & 0.4360       & 0.4527        & 0.4488        & 0.3632       & 0.3493        & 0.3701        & 0.3574       & 0.3418        & 0.3473        \\ \midrule
\textbf{IM}   & 0.4525       & 0.4662        & 0.4705        & 0.3955       & 0.3599        & 0.3800        & 0.3559       & 0.3431        & 0.3488        \\ \midrule
\textbf{NIPM} & 0.4225       & 0.4186        & 0.4673        & 0.2482       & 0.2286        & 0.2319        & 0.0775       & 0.0791        & 0.0742        \\ \midrule
\textbf{IUPM} & {\ul 0.2914} & {\ul 0.2894}  & {\ul 0.3035}  & {\ul 0.0781} & {\ul 0.0793}  & {\ul 0.1020}  & {\ul 0.0352} & {\ul 0.0390}  & {\ul 0.0359}  \\ \midrule
\textbf{$\text{IUPM}_{\textit{UI}}$} &
\textbf{0.0322} &
\textbf{0.0307} &
\textbf{0.0331} &
\textbf{0.0250} &
\textbf{0.0250} &
\textbf{0.0230} &
\textbf{0.0136} &
\textbf{0.0144} &
\textbf{0.0138} \\ \bottomrule
\end{tabular}
\end{table*}

\begin{table*}[htb!]
\centering
\caption{Mean Average Error (MAE) between ground truth and estimated accuracy and number of triggered label interventions ($n_I$) for Random Intervention (RI), Cross Entropy Intervention (CEI), and our proposed Uncertainty Intervention (UI) across three synthetic data sets and three different models.}
\label{tab:sampling_methods}
\resizebox{0.9\textwidth}{!}{\begin{tabular}{@{}ccccccccccc@{}}
\toprule
\multicolumn{2}{c}{\multirow{2}{*}{\textbf{Method}}} &
  \multicolumn{3}{c}{\textbf{Clusters}} &
  \multicolumn{3}{c}{\textbf{Moons}} &
  \multicolumn{3}{c}{\textbf{Circles}} \\ \cmidrule(l){3-11} 
\multicolumn{2}{c}{} &
  \textbf{RF} &
  \textbf{XGB} &
  \textbf{MLP} &
  \textbf{RF} &
  \textbf{XGB} &
  \textbf{MLP} &
  \textbf{RF} &
  \textbf{XGB} &
  \textbf{MLP} \\ \midrule
\multirow{2}{*}{\textbf{$\text{IUPM}_{\textit{RI}}$}} &
  \textbf{MAE} &
  0.0336 &
  0.0327 &
  0.0296 &
  0.0256 &
  \textbf{0.0234} &
  \textbf{0.0198} &
  \textbf{0.0160} &
  0.0177 &
  0.0176 \\
 &
  \textbf{$\bold {n_I}$} &
  32 &
  35 &
  34 &
  23 &
  22 &
  23 &
  31 &
  30 &
  31 \\ \midrule
\multirow{2}{*}{\textbf{$\text{IUPM}_{\textit{CEI}}$}} &
  \textbf{MAE} &
  0.0432 &
  0.0397 &
  0.0456 &
  \textbf{0.0216} &
  0.0258 &
  0.0219 &
  0.0187 &
  0.0206 &
  0.0234 \\
 &
  \textbf{$\bold {n_I}$} &
  35 &
  34 &
  34 &
  22 &
  21 &
  19 &
  29 &
  27 &
  31 \\ \midrule
\multirow{2}{*}{\textbf{$\text{IUPM}_{\textit{UI}}$}} &
  \textbf{MAE} &
  \textbf{0.0270} &
  \textbf{0.0272} &
  \textbf{0.0265} &
  0.0244 &
  0.0242 &
  0.0222 &
  \textbf{0.0160} &
  \textbf{0.0157} &
  \textbf{0.0158} \\
 &
  \textbf{$\bold {n_I}$} &
  \textbf{15} &
  \textbf{15} &
  \textbf{15} &
  \textbf{10} &
  \textbf{10} &
  \textbf{10} &
  \textbf{13} &
  \textbf{13} &
  \textbf{13} \\ \bottomrule
\end{tabular}}
\end{table*}

\begin{figure*}[htb!]
    \centering
    \begin{minipage}[c]{0.31\textwidth}
        \centering
        \resizebox{\textwidth}{!}{%
            \begin{tabular}{@{}cccc@{}}
            \toprule
            \multirow{2}{*}{\textbf{Method}}   & \multirow{2}{*}{\textbf{Rotation}} & \multirow{2}{*}{\textbf{Scaling}} & \multirow{2}{*}{\textbf{Translation}} \\
            &             &             &             \\ \midrule
            \textbf{ATC}  & 0.4836       & 0.1763       & 0.3111       \\ \midrule
            \textbf{AC}   & 0.4931       & 0.2292       & 0.3842       \\ \midrule
            \textbf{DOC}  & 0.5181       & 0.2538       & 0.4090       \\ \midrule
            \textbf{IM}   & 0.6282       & 0.6166       & 0.5686       \\ \midrule
            \textbf{NIPM}  & 0.2187       & 0.0676       & 0.3110       \\ \midrule
            \textbf{IUPM} & {\ul 0.0985} & {\ul 0.0442} & {\ul 0.1263} \\ \midrule
            \textbf{$\text{IUPM}_{\textit{UI}}$} & \textbf{0.0719} & \textbf{0.0438} & \textbf{0.0777} \\ \bottomrule
            \end{tabular}
        }
        \captionof{table}{Mean Average Error (MAE) between ground truth and estimated accuracy for a LeNet across three different shifts on the MNIST data set. The table shows the mean across five random seeds, we refer to \hyperref[app:D]{Appendix D} for confidence intervals.}
        \label{tab:mnist_shifts}
    \end{minipage}
    \hspace{0.01\textwidth}
    \begin{minipage}[c]{0.62\textwidth}

        \begin{subfigure}{.5\textwidth}
        \centering
        \includegraphics[width=\linewidth]{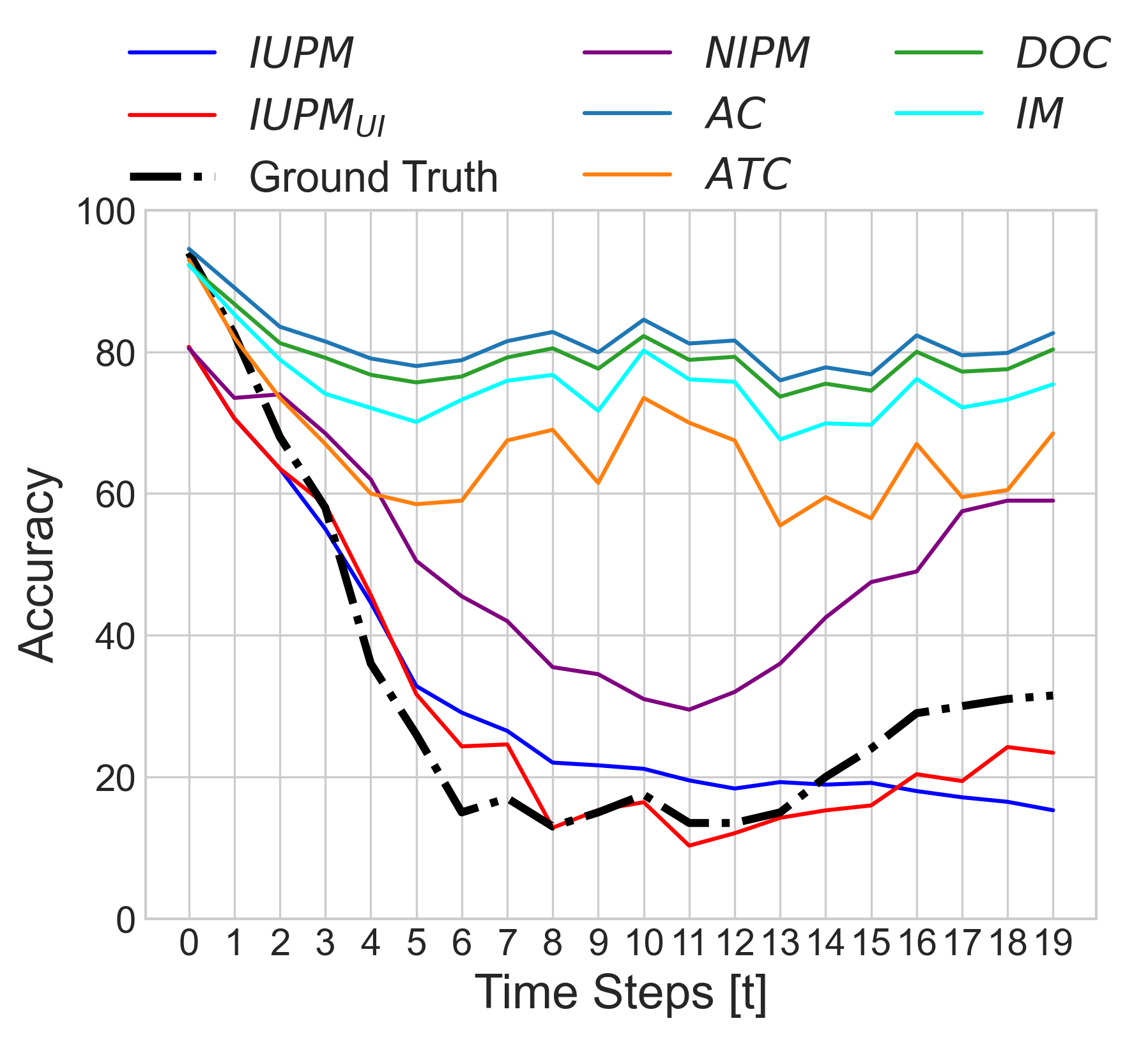}
        \end{subfigure}
        \begin{subfigure}{.5\textwidth}
        \centering
        \includegraphics[width=\linewidth]{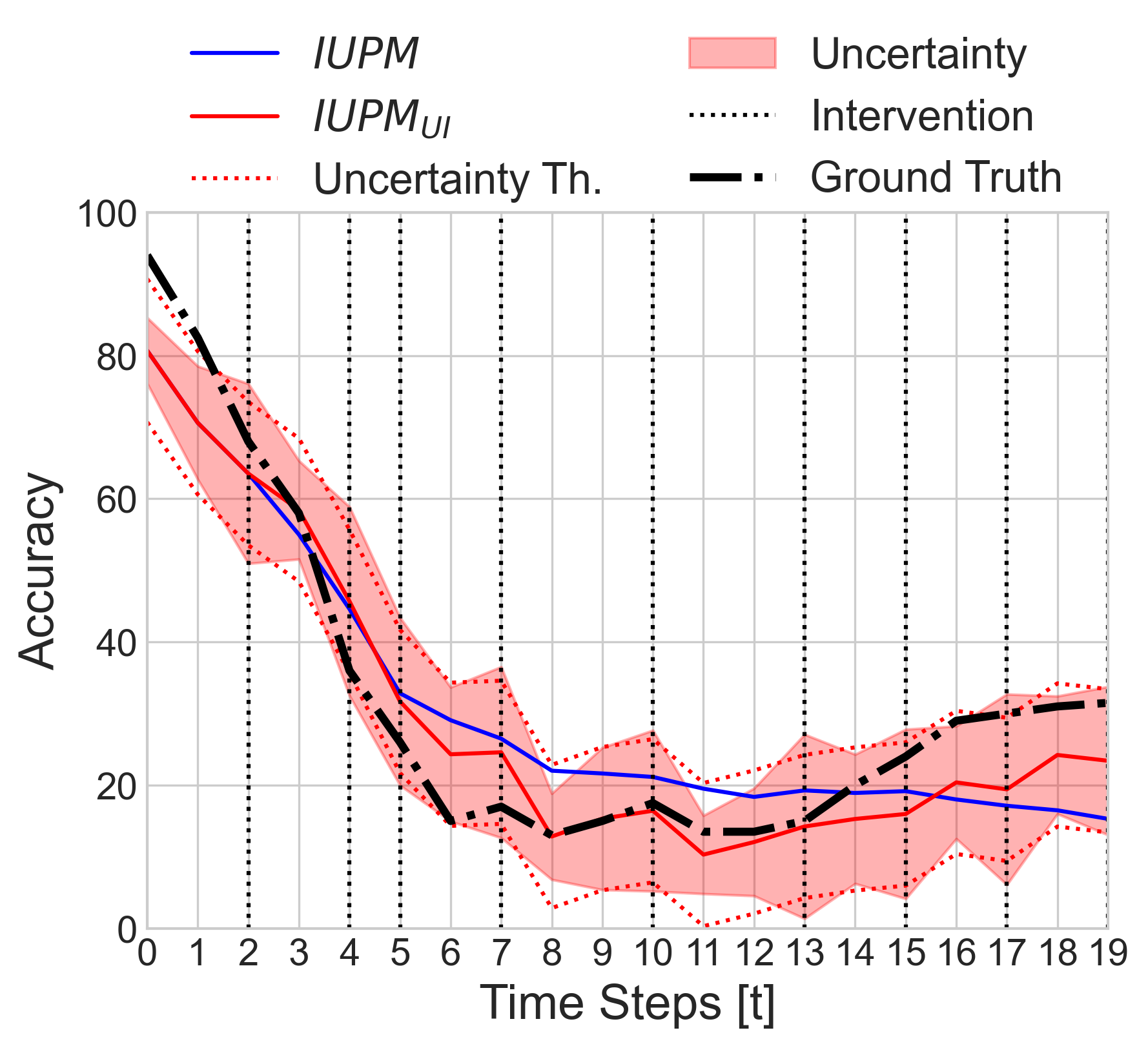}
        \end{subfigure}
        \caption{Performance estimation for a rotational shift on the MNIST digits accumulating to a 180° rotation after 20 steps. Comparing IUPM to the different baselines in (Left) illustrates that it offers the highest fidelity for the estimation. (Right) shows the benefit of including Uncertainty Intervention (UI). The intervention steps triggered by exceeding the threshold are indicated as dashed vertical lines.}
        \label{fig:mnist}
    \end{minipage}
\end{figure*}

\paragraph{Monitoring Performance Degradation due to Image Perturbations}
To make a step towards assessing more complex shifts, we first monitor a model classifying handwritten digits \citep{mnist} that experience a shift caused by affine transformations, i.e., rotation, used in a related context in \citep{wang2020continuously}, translation, and scaling of the digits. For this task, we trained a LeNet \citep{lecun1998gradient} and performed the matching based on representations of the second classification layer of the model, which is a common approach to apply optimal transport to high dimensional data \citep{courty2016optimal, he2023gradual}.
All shifts are evaluated for $20$ steps where we draw a distinct set of $200$ samples from the test set in every step. Figure \ref{fig:mnist} illustrates the performance estimation for a 180° rotation. Even without intervention, IUPM can estimate the performance best. The estimation by NIPM and the other baseline methods can only follow the ground truth sufficiently for a few time steps. The performance estimation error can be further reduced by allowing label intervention with the same settings as in the previous experiment (Figure \ref{fig:mnist}). To keep the uncertainty below the predefined threshold, more interventions are needed in early steps where the shift more strongly impacts the performance.
The observations are further supported by the results on two additional shifts in Table \ref{tab:mnist_shifts}.
To substantiate the findings on MNIST, we have additionally validated our approach on a ResNet-50 \citep{he2016deep} classifying 500 samples from the Imagnette validation dataset comprising 10 classes of ImageNet. We analyze several shifts from \texttt{ImageNet-c} \citep{hendrycks2018benchmarking} that can be considered gradual and further interpolate them over 20 steps. Again, IUPM gives on average the best label-free performance estimate across all considered methods, while only falling short on two shifts. Moreover, relying on our Uncertainty-based labeling interventions ($\text{IUPM}_{\textit{UI}}$) the estimate can further be improved by around $10\%$.
\begin{table*}[h!]
\centering
\begin{tabular}{lccccccc|c}
\toprule
    \textbf{Method}   & \textbf{Blur} & \textbf{Contrast} & \textbf{Brightness} & \textbf{Rotation} & \textbf{Scale} & \textbf{Shear} & \textbf{Translate} & \textbf{Mean }\\
\midrule
\textbf{ATC}     & 0.0719 & 0.0562 & 0.0260 & \textbf{0.0564} & 0.0871 & 0.0825 & 0.0787 & 0.0655 \\
\midrule
\textbf{AC}      & 0.0617 & 0.0825 & 0.0551 & 0.1017 & 0.1346 & 0.1695 & 0.0569 & 0.0946 \\
\midrule
\textbf{DOC}     & 0.0608 & 0.0827 & 0.0574 & 0.1055 & 0.1349 & 0.1700 & 0.0565 & 0.0954 \\
\midrule
\textbf{IM}      & \textbf{0.0456} & 0.0936 & 0.0620 & 0.1320 & 0.1166 & 0.1737 & 0.0500 & 0.0962 \\
\midrule
\textbf{NIPM}    & 0.1094 & 0.0805 & 0.0697 & 0.1117 & 0.2127 & 0.2327 & 0.1254 & 0.1346 \\
\midrule
\textbf{IUPM}    & 0.0610 & \textbf{0.0377} & \underline{0.0208} & 0.1024 & \underline{0.0766} & \underline{0.0341} & \underline{0.0469} & \underline{0.0542} \\
\midrule
\textbf{$\text{IUPM}_{\textit{UI}}$} & \underline{0.0569} & \underline {0.0417} & \textbf{0.0173} & \underline{0.0712} & \textbf{0.0752} & \textbf{0.0339} & \textbf{0.0469} & \textbf{0.0490} \\
\bottomrule
\end{tabular}
\caption{Mean Average Errors (MAE) between ground truth and estimated accuracy for a ResNet-50 across different gradual shifts from \texttt{ImageNet-c}. The results indicate that IUPM gives on average the best label-free performance estimate which can be further improved using our labeling interventions $\text{IUPM}_{\textit{UI}}$.}
\label{tab:imagenet}
\end{table*}
\paragraph{Monitoring Performance Degradation due to Real-World Temporal Shifts}
In this experiment, we show the applicability of our approach in a real-world shift scenario. For this, we utilize a gender classification data set consisting of yearbook portraits \citep{ginosar2015century} across decades, which is commonly used to evaluate methods with respect to gradual shifts \citep{yao2022wild, kumar2020understanding, he2023gradual}.
We trained a simple convolutional network on the available portraits from 1930 to 1934.
By using the samples from 1935 to initialize both our IUPM approach and the baseline methods, we evaluate the performance decline from 1936 to 2013. Unlike previous experiments on synthetic shifts presenting a continuous decline in model performance, Figure \ref{fig:portraits} shows that the ground truth performance of the model remains almost constant until 1966, when a severe dip in the model performance occurs which partly recovers until 1996.
\begin{figure*}
\centering
\begin{subfigure}{.48\textwidth}
  \centering
  \includegraphics[width=\linewidth]{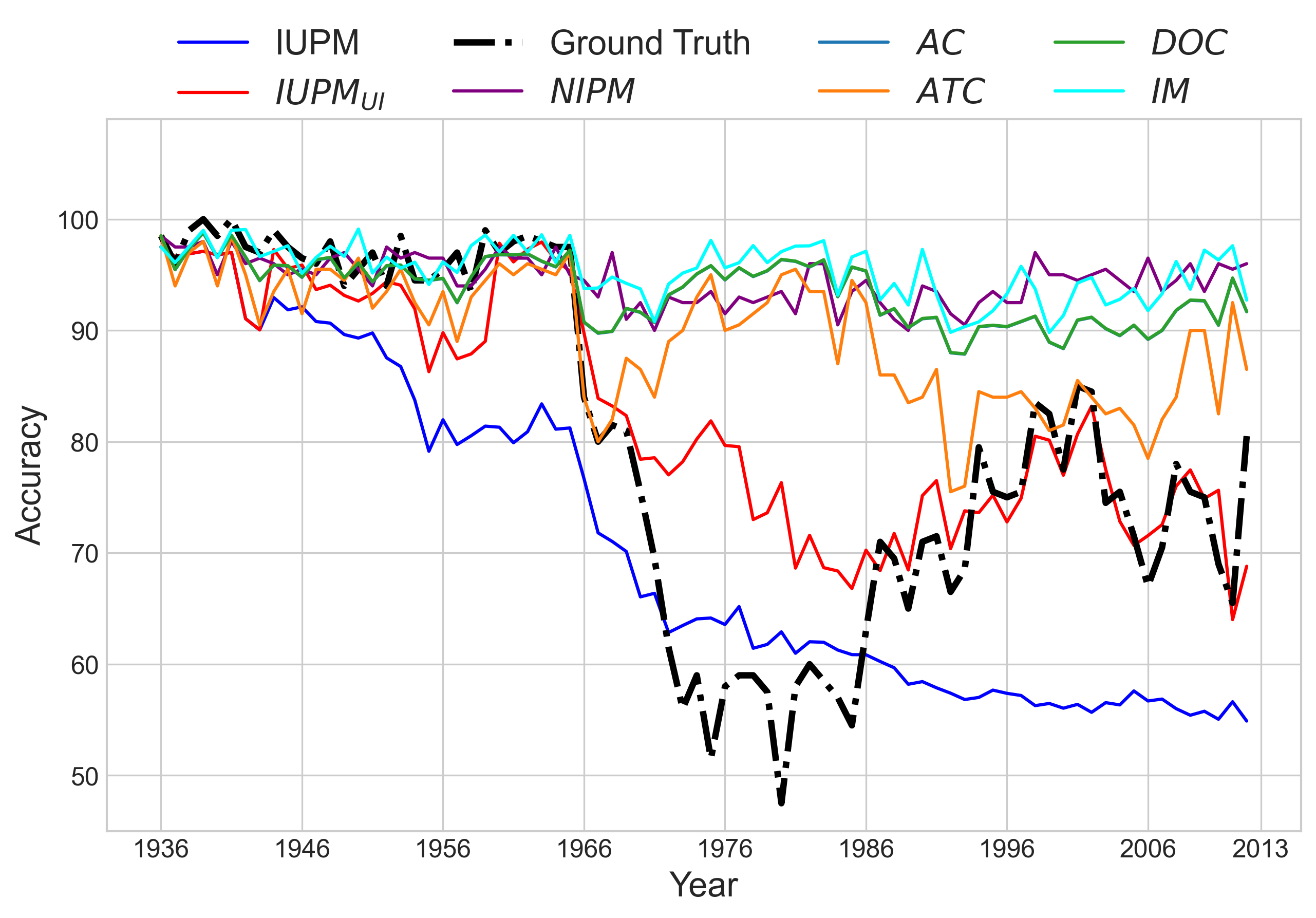}
\end{subfigure}
\begin{subfigure}{.48\textwidth}
  \centering
  \includegraphics[width=\linewidth]{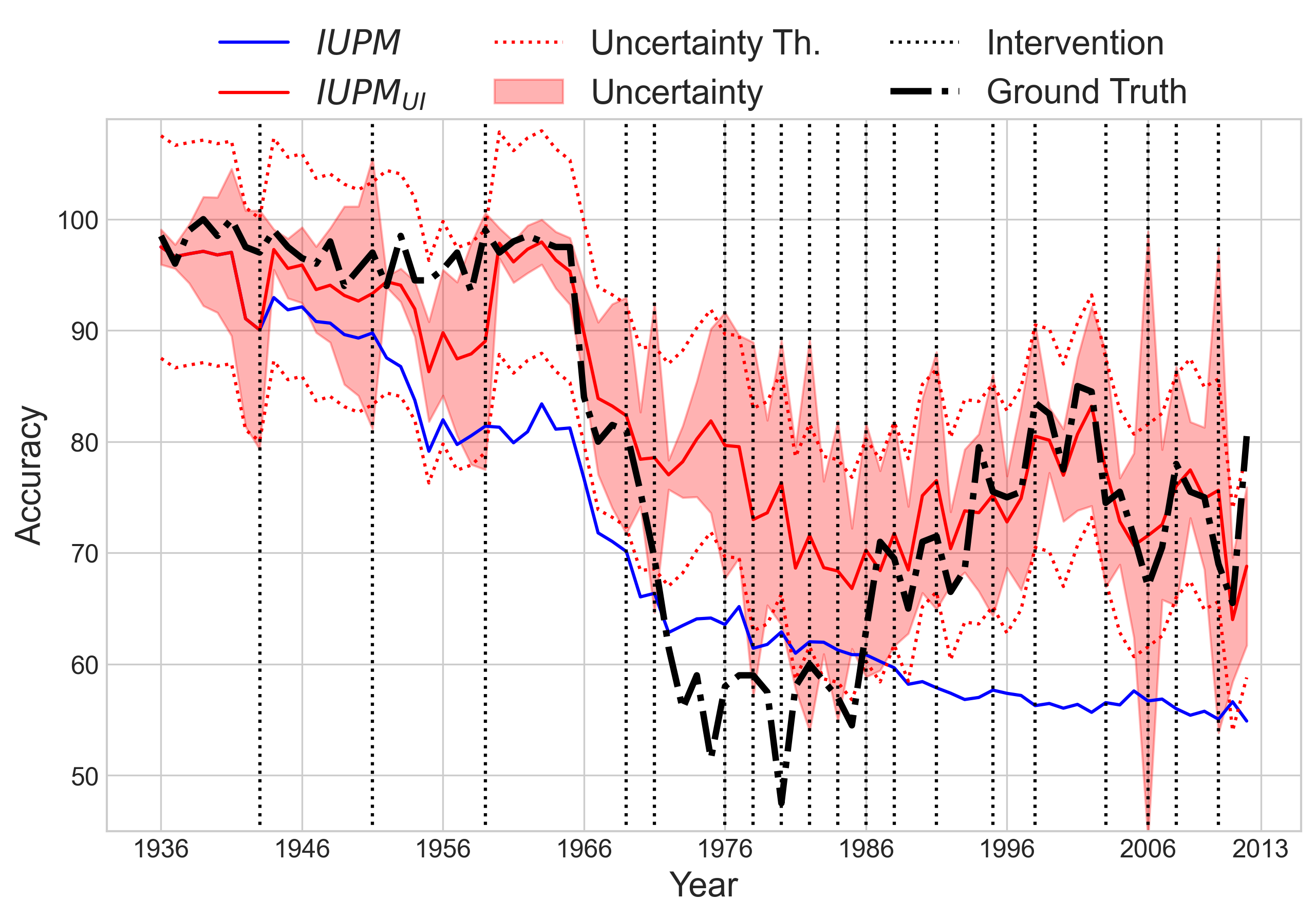}
\end{subfigure}
\caption{Performance estimation for yearbook model across portraits from 1936 to 2013. The comparison with baselines in (Left) shows that only the estimation by IUPM is consistent with the dip in performance between 1966 and 1996. (Right) illustrates the benefit of including Uncertainty Intervention (UI). The intervention steps triggered by exceeding the threshold are indicated as dashed vertical lines.}
\label{fig:portraits}
\end{figure*}
For complex real-world shifts such as the appearance of yearbook photos, the decline in model performance might be linked to various reasons. One such factor fitting very well to the observed performance degradation is quite evident in Figure \ref{fig:fig1}, showing that the hairstyle of male students around 1980 developed to be quite similar to a typical female hairstyle around 1930, which is considered in the models training set. IUPM is the only method that correctly captures this performance drop. Through the introduction of active label intervention steps, our approach is also capable of following a subsequent increase in the model's performance (Figure \ref{fig:portraits}). Further, our method correctly identifies areas with substantial changes and resulting high uncertainties in the performance estimation increasing the frequency of triggering label interventions. In real-world applications, it is highly desirable to identify such time periods to increase the labeling effort rather than relying on a non-faithful estimate.
\paragraph{Empirical Estimation of Gradual Lipschitz Smoothness}
As a final experiment, we demonstrate how to empirically asses if the underlying assumption of a shift being gradually Lipschitz smooth (Definition 2) holds practice. Again, we consider the portraits classification dataset comprising real college portraits, and we use only a small number of 100 samples per step. To get an estimate $\hat{\varepsilon}$ of the Wasserstein distance $\mathcal{W}$ between two input distributions at consecutive time steps, we solve the underlying transport problem using a linear program solver \citep{peyre2019computational}. To scale this approach to high-dimensional real datasets, we quantify this distributional distance based on corresponding network activations of the penultimate layer instead of the raw input, which is a common practice to compare realistic data distribution \citep{heusel2017gans,zhang2018unreasonable}. 
Moreover, since every input image has an unambiguous class label, $\mathcal{W}(P(Y_t|X_t), P(Y_{t-1}|X_{t-1})) = 1$ for samples that have different labels and 0 otherwise. Hence, one can verify the Lipschitz property at time $t$ by identifying the pair $(x_t^*, x_{t-1}^*)$ with different labels but minimal distance within the available samples:
\begin{align*}
(x_t^*, x_{t-1}^*) = \arg \min_{x_t , x_{t-1}} c(x_t, x_{t-1}) \quad \text{st.} \quad y_t \ne y_{t-1}.
\end{align*}
This results in:
\begin{align*}
\dfrac{\mathcal{W}(P(Y_t|X_t), P(Y_{t-1}|X_{t-1})) }{ c(x_t, x_{t-1}) }  \le \dfrac{1}{c(x_t^*, x_{t-1}^*)} \le L_t.
\end{align*}
Therefore, we can use $\hat{L}_t:= 1/c(x_t^*, x_{t-1}^*) $ as an empirical lower bound for Lipschitzness at time $t$. To validate the utility of such estimates, we performed two complementary correlation studies analyzing two theoretical relationships. First, Proposition 1 shows that the constants $L_t$ and $\varepsilon_t$ determine the strength of the overall shift through $(1+L_t)\varepsilon_t$. If our estimates are meaningful, this quantity should correlate with the actual performance $\mathcal{L}_t$. Second, Theorem 1 provides an upper bound on the performance estimation error $|\mathcal{L}_t- \hat{\mathcal{L}}_t^{\textit{IUPM}} |$ in terms of these constants. Hence, we would expect a correspondence between this theoretical bound and the actual error if our empirical estimates capture the relevant properties of the observed shifts.
\begin{table}[htb!]
\centering
\caption{Pearson Correlation ($\rho$) between performance ($\mathcal{L}_t$) and estimation error ($|\mathcal{L}_t- \hat{\mathcal{L}}_t^{\textit{IUPM}} |$) at time $t$ with corresponding theoretical upper bounds from Proposition 1 and Theorem 1. The strong linear correlation implies that estimating the underlying quantities can yield a useful indication of the extent to which an observed shift is gradually Lipschitz smooth.}
\label{tab:correlation}
\begin{tabular}{lcc}
\toprule
Quantity vs Proxy Estimate & $\rho$ & p-value \\
\midrule
\hspace{1.6cm}$\mathcal{L}_t$ vs. $(1+\hat{L}_t)\hat{\varepsilon_t}$ & 0.86 & $< 0.001$ \\
$|\mathcal{L}_t- \hat{\mathcal{L}}_t^{\textit{IUPM}} |$  vs. $\sum_{i=0}^t\hat{L}_i\hat{\varepsilon_t}$ & 0.56 & $< 0.001$ \\
\bottomrule
\end{tabular}
\end{table}
Both results indicate a high linear correlation with strong statistical significance between the estimated quantities and two related theoretical expressions. This provides evidence that estimating the underlying quantities can yield a useful indication if an observed shift is gradually Lipschitz smooth.
\begin{figure}[htb!]
    \centering
    \includegraphics[width=1\linewidth]{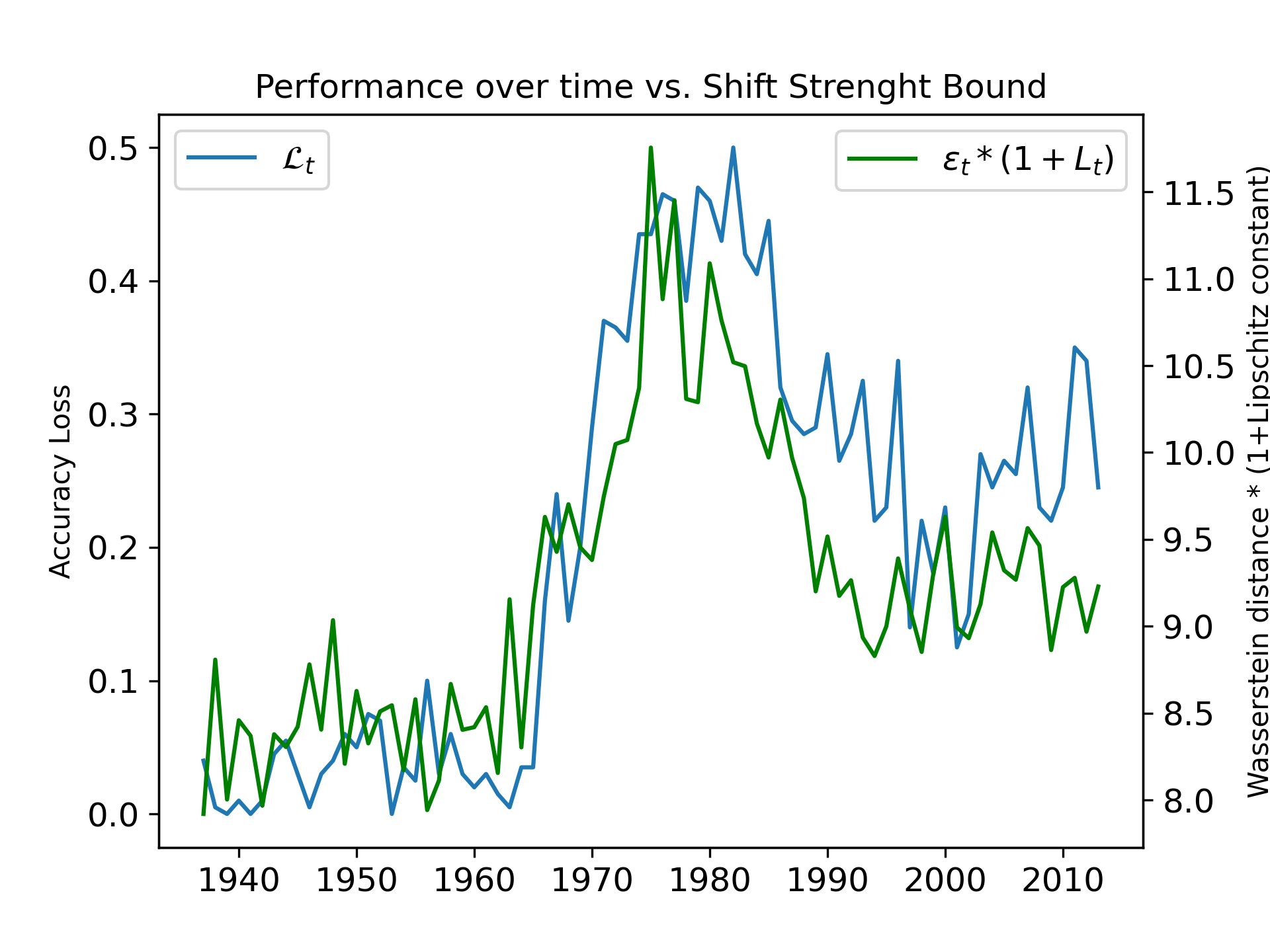}
    \caption{Performance ($\mathcal{L}_t$) and estimated overall shift strength ($(1+\hat{L}_t)\hat{\varepsilon}_t$) over time $t$.
    }
   \vspace{-0.4cm}
\end{figure}
\section{CONCLUSION}
In this work, we introduce IUPM, a novel method designed to anticipate the performance of deployed machine learning models under gradual changes over time. We theoretically analyze the underlying assumptions and demonstrate that its estimates closely align with true performance for a broad class of gradual distribution shifts. Additionally, IUPM naturally quantifies uncertainty, enabling more reliable assessments by actively querying for labels to enhance the trustworthiness of its estimates when necessary. Through analysis of simple synthetic datasets, we illustrate the underlying principles. We further validate its effectiveness on both simulated and real gradual and temporal shifts. While IUPM is specifically tailored for gradual changes, it may not be optimal for other types of distribution shifts. When these distribution shifts result in a high uncertainty of the sample-wise loss estimates, our UI approach may require frequent interventions. 
In many real-world cases, however, providing additional labeled data more frequently is preferable rather than relying on an inaccurate performance estimate without any indication of uncertainty at all. 

\bibliographystyle{apalike}
\bibliography{bibliography}

\begin{thebibliography}{}

\bibitem[Abnar et~al., 2021]{abnar2021gradual}
Abnar, S., Berg, R. v.~d., Ghiasi, G., Dehghani, M., Kalchbrenner, N., and Sedghi, H. (2021).
\newblock Gradual domain adaptation in the wild: When intermediate distributions are absent.
\newblock {\em arXiv preprint arXiv:2106.06080}.

\bibitem[Baek et~al., 2022]{baek2022agreementontheline}
Baek, C., Jiang, Y., Raghunathan, A., and Kolter, J.~Z. (2022).
\newblock Agreement-on-the-line: Predicting the performance of neural networks under distribution shift.
\newblock In Oh, A.~H., Agarwal, A., Belgrave, D., and Cho, K., editors, {\em Advances in Neural Information Processing Systems}.

\bibitem[Chen and Chao, 2021]{chen2021gradual}
Chen, H.-Y. and Chao, W.-L. (2021).
\newblock Gradual domain adaptation without indexed intermediate domains.
\newblock {\em Advances in neural information processing systems}, 34:8201--8214.

\bibitem[Chen et~al., 2021]{chen2021mandoline}
Chen, M., Goel, K., Sohoni, N.~S., Poms, F., Fatahalian, K., and R{\'e}, C. (2021).
\newblock Mandoline: Model evaluation under distribution shift.
\newblock In {\em International conference on machine learning}, pages 1617--1629. PMLR.

\bibitem[Chen and Guestrin, 2016]{chen2016xgboost}
Chen, T. and Guestrin, C. (2016).
\newblock Xgboost: A scalable tree boosting system.
\newblock In {\em Proceedings of the 22nd acm sigkdd international conference on knowledge discovery and data mining}, pages 785--794.

\bibitem[Courty et~al., 2017]{courty2017joint}
Courty, N., Flamary, R., Habrard, A., and Rakotomamonjy, A. (2017).
\newblock Joint distribution optimal transportation for domain adaptation.
\newblock {\em Advances in neural information processing systems}, 30.

\bibitem[Courty et~al., 2016]{courty2016optimal}
Courty, N., Flamary, R., Tuia, D., and Rakotomamonjy, A. (2016).
\newblock Optimal transport for domain adaptation.
\newblock {\em IEEE transactions on pattern analysis and machine intelligence}, 39(9):1853--1865.

\bibitem[Cuturi, 2013]{cuturi2013sinkhorn}
Cuturi, M. (2013).
\newblock Sinkhorn distances: Lightspeed computation of optimal transport.
\newblock {\em Advances in neural information processing systems}, 26.

\bibitem[Decker et~al., 2024]{decker2024explanatory}
Decker, T., Koebler, A., Lebacher, M., Thon, I., Tresp, V., and Buettner, F. (2024).
\newblock Explanatory model monitoring to understand the effects of feature shifts on performance.
\newblock In {\em Proceedings of the 30th ACM SIGKDD Conference on Knowledge Discovery and Data Mining}, pages 550--561.

\bibitem[Deng et~al., 2021]{deng2021does}
Deng, W., Gould, S., and Zheng, L. (2021).
\newblock What does rotation prediction tell us about classifier accuracy under varying testing environments?
\newblock In {\em International Conference on Machine Learning}, pages 2579--2589. PMLR.

\bibitem[Deng and Zheng, 2021]{deng2021labels}
Deng, W. and Zheng, L. (2021).
\newblock Are labels always necessary for classifier accuracy evaluation?
\newblock In {\em Proceedings of the IEEE/CVF Conference on Computer Vision and Pattern Recognition}, pages 15069--15078.

\bibitem[Everitt, 2013]{everitt2013finite}
Everitt, B. (2013).
\newblock {\em Finite mixture distributions}.
\newblock Springer Science \& Business Media.

\bibitem[Flamary et~al., 2021]{flamary2021pot}
Flamary, R., Courty, N., Gramfort, A., Alaya, M.~Z., Boisbunon, A., Chambon, S., Chapel, L., Corenflos, A., Fatras, K., Fournier, N., Gautheron, L., Gayraud, N.~T., Janati, H., Rakotomamonjy, A., Redko, I., Rolet, A., Schutz, A., Seguy, V., Sutherland, D.~J., Tavenard, R., Tong, A., and Vayer, T. (2021).
\newblock Pot: Python optimal transport.
\newblock {\em Journal of Machine Learning Research}, 22(78):1--8.

\bibitem[Gama et~al., 2014]{gama2014survey}
Gama, J., {\v{Z}}liobait{\.e}, I., Bifet, A., Pechenizkiy, M., and Bouchachia, A. (2014).
\newblock A survey on concept drift adaptation.
\newblock {\em ACM computing surveys (CSUR)}, 46(4):1--37.

\bibitem[Garg et~al., 2022]{gargleveraging}
Garg, S., Balakrishnan, S., Lipton, Z.~C., Neyshabur, B., and Sedghi, H. (2022).
\newblock Leveraging unlabeled data to predict out-of-distribution performance.
\newblock In {\em International Conference on Learning Representations}.

\bibitem[Garg et~al., 2020]{garg2020unified}
Garg, S., Wu, Y., Balakrishnan, S., and Lipton, Z. (2020).
\newblock A unified view of label shift estimation.
\newblock {\em Advances in Neural Information Processing Systems}, 33:3290--3300.

\bibitem[Genevay et~al., 2019]{genevay2019sample}
Genevay, A., Chizat, L., Bach, F., Cuturi, M., and Peyr{\'e}, G. (2019).
\newblock Sample complexity of sinkhorn divergences.
\newblock In {\em The 22nd international conference on artificial intelligence and statistics}, pages 1574--1583. PMLR.

\bibitem[Ginosar et~al., 2015]{ginosar2015century}
Ginosar, S., Rakelly, K., Sachs, S., Yin, B., and Efros, A.~A. (2015).
\newblock A century of portraits: A visual historical record of american high school yearbooks.
\newblock In {\em Proceedings of the IEEE International Conference on Computer Vision Workshops}, pages 1--7.

\bibitem[Gozlan and L{\'e}onard, 2010]{gozlan2010transport}
Gozlan, N. and L{\'e}onard, C. (2010).
\newblock Transport inequalities. a survey.
\newblock {\em arXiv preprint arXiv:1003.3852}.

\bibitem[Guillory et~al., 2021]{guillory2021predicting}
Guillory, D., Shankar, V., Ebrahimi, S., Darrell, T., and Schmidt, L. (2021).
\newblock Predicting with confidence on unseen distributions.
\newblock In {\em Proceedings of the IEEE/CVF International Conference on Computer Vision}, pages 1134--1144.

\bibitem[He et~al., 2016]{he2016deep}
He, K., Zhang, X., Ren, S., and Sun, J. (2016).
\newblock Deep residual learning for image recognition.
\newblock In {\em Proceedings of the IEEE conference on computer vision and pattern recognition}, pages 770--778.

\bibitem[He et~al., 2024]{he2023gradual}
He, Y., Wang, H., Li, B., and Zhao, H. (2024).
\newblock Gradual domain adaptation: Theory and algorithms.
\newblock {\em Journal of Machine Learning Research}, 25(361):1--40.

\bibitem[Hendrycks and Dietterich, 2019]{hendrycks2018benchmarking}
Hendrycks, D. and Dietterich, T. (2019).
\newblock Benchmarking neural network robustness to common corruptions and perturbations.
\newblock In {\em International Conference on Learning Representations}.

\bibitem[Heusel et~al., 2017]{heusel2017gans}
Heusel, M., Ramsauer, H., Unterthiner, T., Nessler, B., and Hochreiter, S. (2017).
\newblock Gans trained by a two time-scale update rule converge to a local nash equilibrium.
\newblock {\em Advances in neural information processing systems}, 30.

\bibitem[Howard, 2019]{Howard_Imagenette_2019}
Howard, J. (2019).
\newblock Imagenette: A smaller subset of 10 easily classified classes from imagenet.

\bibitem[Jiang et~al., 2022]{jiang2022assessing}
Jiang, Y., Nagarajan, V., Baek, C., and Kolter, J.~Z. (2022).
\newblock Assessing generalization of {SGD} via disagreement.
\newblock In {\em International Conference on Learning Representations}.

\bibitem[Kossen et~al., 2021]{kossen2021active}
Kossen, J., Farquhar, S., Gal, Y., and Rainforth, T. (2021).
\newblock Active testing: Sample-efficient model evaluation.
\newblock In {\em International Conference on Machine Learning}, pages 5753--5763. PMLR.

\bibitem[Kossen et~al., 2022]{kossen2022active}
Kossen, J., Farquhar, S., Gal, Y., and Rainforth, T. (2022).
\newblock Active surrogate estimators: An active learning approach to label-efficient model evaluation.
\newblock {\em Advances in Neural Information Processing Systems}, 35:24557--24570.

\bibitem[Kumar et~al., 2020]{kumar2020understanding}
Kumar, A., Ma, T., and Liang, P. (2020).
\newblock Understanding self-training for gradual domain adaptation.
\newblock In {\em International conference on machine learning}, pages 5468--5479. PMLR.

\bibitem[Lecun et~al., 1998]{mnist}
Lecun, Y., Bottou, L., Bengio, Y., and Haffner, P. (1998).
\newblock Gradient-based learning applied to document recognition.
\newblock {\em Proceedings of the IEEE}, 86(11):2278--2324.

\bibitem[LeCun et~al., 1998]{lecun1998gradient}
LeCun, Y., Bottou, L., Bengio, Y., and Haffner, P. (1998).
\newblock Gradient-based learning applied to document recognition.
\newblock {\em Proceedings of the IEEE}, 86(11):2278--2324.

\bibitem[Lee et~al., 2024]{Lee2024TowardsOM}
Lee, J., Kolla, L., and Chen, J. (2024).
\newblock Towards optimal model evaluation: enhancing active testing with actively improved estimators.
\newblock {\em Scientific Reports}, 14.

\bibitem[Lu et~al., 2023]{lu2023characterizing}
Lu, Y., Qin, Y., Zhai, R., Shen, A., Chen, K., Wang, Z., Kolouri, S., Stepputtis, S., Campbell, J., and Sycara, K. (2023).
\newblock Characterizing out-of-distribution error via optimal transport.
\newblock {\em Advances in Neural Information Processing Systems}, 36:17602--17622.

\bibitem[Mu and Gilmer, 2019]{mu2019mnist}
Mu, N. and Gilmer, J. (2019).
\newblock Mnist-c: A robustness benchmark for computer vision.
\newblock {\em arXiv preprint arXiv:1906.02337}.

\bibitem[Pedregosa et~al., 2011]{scikit-learn}
Pedregosa, F., Varoquaux, G., Gramfort, A., Michel, V., Thirion, B., Grisel, O., Blondel, M., Prettenhofer, P., Weiss, R., Dubourg, V., Vanderplas, J., Passos, A., Cournapeau, D., Brucher, M., Perrot, M., and Duchesnay, E. (2011).
\newblock Scikit-learn: Machine learning in {P}ython.
\newblock {\em Journal of Machine Learning Research}, 12:2825--2830.

\bibitem[Peyr{\'e} et~al., 2019]{peyre2019computational}
Peyr{\'e}, G., Cuturi, M., et~al. (2019).
\newblock Computational optimal transport: With applications to data science.
\newblock {\em Foundations and Trends{\textregistered} in Machine Learning}, 11(5-6):355--607.

\bibitem[Sawade et~al., 2010]{sawade2010active}
Sawade, C., Landwehr, N., Bickel, S., and Scheffer, T. (2010).
\newblock Active risk estimation.
\newblock In {\em Proceedings of the 27th International Conference on International Conference on Machine Learning}, ICML'10, page 951–958, Madison, WI, USA. Omnipress.

\bibitem[Sugiyama et~al., 2008]{sugiyama2008direct}
Sugiyama, M., Suzuki, T., Nakajima, S., Kashima, H., von B{\"u}nau, P., and Kawanabe, M. (2008).
\newblock Direct importance estimation for covariate shift adaptation.
\newblock {\em Annals of the Institute of Statistical Mathematics}, 60(4):699--746.

\bibitem[Villani et~al., 2009]{villani2009optimal}
Villani, C. et~al. (2009).
\newblock {\em Optimal transport: old and new}, volume 338.
\newblock Springer.

\bibitem[Wang et~al., 2020]{wang2020continuously}
Wang, H., He, H., and Katabi, D. (2020).
\newblock Continuously indexed domain adaptation.
\newblock In {\em Proceedings of the 37th International Conference on Machine Learning}, pages 9898--9907.

\bibitem[Wang et~al., 2022]{wang2022understanding}
Wang, H., Li, B., and Zhao, H. (2022).
\newblock Understanding gradual domain adaptation: Improved analysis, optimal path and beyond.
\newblock In {\em International Conference on Machine Learning}, pages 22784--22801. PMLR.

\bibitem[Wilson and Cook, 2020]{wilson2020survey}
Wilson, G. and Cook, D.~J. (2020).
\newblock A survey of unsupervised deep domain adaptation.
\newblock {\em ACM Transactions on Intelligent Systems and Technology (TIST)}, 11(5):1--46.

\bibitem[Xie et~al., 2024]{xie2024evolving}
Xie, M., Li, S., Yuan, L., Liu, C., and Dai, Z. (2024).
\newblock Evolving standardization for continual domain generalization over temporal drift.
\newblock {\em Advances in Neural Information Processing Systems}, 36.

\bibitem[Yao et~al., 2022]{yao2022wild}
Yao, H., Choi, C., Cao, B., Lee, Y., Koh, P. W.~W., and Finn, C. (2022).
\newblock Wild-time: A benchmark of in-the-wild distribution shift over time.
\newblock {\em Advances in Neural Information Processing Systems}, 35:10309--10324.

\bibitem[Yu et~al., 2024]{yu2024survey}
Yu, H., Liu, J., Zhang, X., Wu, J., and Cui, P. (2024).
\newblock A survey on evaluation of out-of-distribution generalization.
\newblock {\em arXiv preprint arXiv:2403.01874}.

\bibitem[Yu et~al., 2022]{yu2022predicting}
Yu, Y., Yang, Z., Wei, A., Ma, Y., and Steinhardt, J. (2022).
\newblock Predicting out-of-distribution error with the projection norm.
\newblock In {\em International Conference on Machine Learning}, pages 25721--25746. PMLR.

\bibitem[Zhang et~al., 2018]{zhang2018unreasonable}
Zhang, R., Isola, P., Efros, A.~A., Shechtman, E., and Wang, O. (2018).
\newblock The unreasonable effectiveness of deep features as a perceptual metric.
\newblock In {\em Proceedings of the IEEE conference on computer vision and pattern recognition}, pages 586--595.

\end{thebibliography}
 \section*{Checklist}

 \begin{enumerate}

 \item For all models and algorithms presented, check if you include:
 \begin{enumerate}
   \item A clear description of the mathematical setting, assumptions, algorithm, and/or model. [Yes]
   \item An analysis of the properties and complexity (time, space, sample size) of any algorithm. [Yes]
   \item (Optional) Anonymized source code, with specification of all dependencies, including external libraries. [Yes, we include a GitHub link]
 \end{enumerate}

 \item For any theoretical claim, check if you include:
 \begin{enumerate}
   \item Statements of the full set of assumptions of all theoretical results. [Yes]
   \item Complete proofs of all theoretical results. [Yes]
   \item Clear explanations of any assumptions. [Yes]     
 \end{enumerate}

 \item For all figures and tables that present empirical results, check if you include:
 \begin{enumerate}
   \item The code, data, and instructions needed to reproduce the main experimental results (either in the supplemental material or as a URL). [Yes]
   \item All the training details (e.g., data splits, hyperparameters, how they were chosen). [Yes]
         \item A clear definition of the specific measure or statistics and error bars (e.g., with respect to the random seed after running experiments multiple times). [Yes]
         \item A description of the computing infrastructure used. (e.g., type of GPUs, internal cluster, or cloud provider). [Yes]
 \end{enumerate}

 \item If you are using existing assets (e.g., code, data, models) or curating/releasing new assets, check if you include:
 \begin{enumerate}
   \item Citations of the creator If your work uses existing assets. [Yes]
   \item The license information of the assets, if applicable. [Not Applicable]
   \item New assets either in the supplemental material or as a URL, if applicable. [Not Applicable]
   \item Information about consent from data providers/curators. [Not Applicable]
   \item Discussion of sensible content if applicable, e.g., personally identifiable information or offensive content. [Not Applicable]
 \end{enumerate}

 \item If you used crowdsourcing or conducted research with human subjects, check if you include:
 \begin{enumerate}
   \item The full text of instructions given to participants and screenshots. [Not Applicable]
   \item Descriptions of potential participant risks, with links to Institutional Review Board (IRB) approvals if applicable. [Not Applicable]
   \item The estimated hourly wage paid to participants and the total amount spent on participant compensation. [Not Applicable]
 \end{enumerate}

 \end{enumerate}
\appendix
\onecolumn
\setcounter{theorem}{0}
\setcounter{definition}{1}
\setcounter{proposition}{0}
\setcounter{corollary}{0}
\setcounter{algorithm}{0}
\section{THEORETICAL PROOFS}
\label{app:A}
In this section, we provide additional details on all theoretical results and conduct missing proofs.
To prove our main result stated in Theorem 1, we rely on the assumption that the experienced shift is gradually Lipschitz smooth, which we restate below:
\begin{definition} A distribution shift over $\{(X_t, Y_t)\}_{t=0}^T$ is called gradually Lipschitz smooth in $X_t$ if for a cost function $c: \mathcal{X} \times \mathcal{X} \rightarrow \mathbb{R}^+$ we have $\mathcal{W}(P_t(X_t), P_{t-1}(X_{t-1})) \le \varepsilon_t $ for all $t=1, \dots , T$ and there exist constants $L_t>0$ such that for any realizations $x_t$, $x_{t-1}$ of $X_t, X_{t-1}$ it holds:
\begin{align*}
    \mathcal{W}\big(P_t(Y_t|X_t=x_t), P_{t-1}(Y_{t-1}|X_{t-1}=x_{t-1})\big)\le L_t \; c(x_t, x_{t-1})
\end{align*}
\end{definition}
Any shift satisfying this property is also gradual (see Definition 1 in the main paper) in the typically assumed sense:

\begin{proposition}
If a distribution shift over $\{(X_t, Y_t)\}_{t=0}^T$ is gradually Lipschitz smooth in $X_t$ with constants $L_t$, then it is also gradual:
\begin{align*}
    \mathcal{W}(P_t(Y_t, X_t), P_{t-1}(Y_{t-1}, X_{t-1})) \le (1+L_t)\;\varepsilon_t := \Delta_t
\end{align*}
\end{proposition}
\begin{proof}
Measuring the Wasserstein distance between two joint distributions requires a cost metric $c_{xy}$ operating on the product space $(\mathcal{X} \times \mathcal{Y})$. Note that if $c_x$ is a cost metric on $\mathcal{X}$ and $c_y$ on $\mathcal{Y}$, a natural choice is to simply consider $c_{xy}$ to be separable:  $c_{xy}((x,y), (x',y')) = c_x(x, x')+ c_y(y, y')$.
Let $\Pi\big(P_t(Y_t, X_t), P_{t-1}(Y_{t-1}, X_{t-1})\big)$ be the space of all valid couplings over the joint distributions of two consecutive domains. Furthermore let $\Pi_x\big(P_t(X_t), P_{t-1}(X_{t-1})\big)$ bet the set of all couplings over marginals in $X$ and given an tuple $x=(x_t, x_{t-1}$), $\Pi_{y|x}\big(P_t(Y_t|X_t=x_t), P_{t-1}(Y_{t-1}|X_{t-1}=x_{t-1})\big)$ be the set of couplings over all target conditionals. Then we have:
\begin{align*}
    \mathcal{W}^{c_{xy}}(P_t(Y_t, X_t), P_{t-1}(Y_{t-1}, X_{t-1})) &= \inf_{\pi \in \Pi} \mathbb{E}_{\pi }\big[  c_x(x_t, x_{t-1}) + c_y(y_t, y_{t-1})]\\
    & \le \inf_{\pi_x \in \Pi_x} \mathbb{E}_{(x_t, x_{t-1}) \sim \pi_x(X_t, X_{t-1}) }\big[ c_x(x_t, x_{t-1}) \big] \quad +  \\ & \qquad \mathbb{E}_{(x_t, x_{t-1}) \sim \pi_x(X_t, X_{t-1}) }\Big[ \inf_{\pi_{y|x} \in \Pi_{y|x}} \mathbb{E}_{(y_t, y_{t-1}) \sim \pi_{y|x}(y_t, y_{t-1} |x_t, x_{t-1}) }\big[ c_x(y_t, y_{t-1}) \big]  \Big] \\
    & = \mathcal{W}^{c_{x}}(P_t(X_t), P_{t-1}(X_{t-1})) \quad + \\ &\qquad \mathbb{E}_{(x_t, x_{t-1}) \sim \pi(x_t, x_{t-1}) }\big[ \mathcal{W}^{c_y}(P_t(Y_t |X_t = x_t), P_{t-1}(Y_{t-1} | X_{t-1}=x_{t-1})) \big] \\
    &\le \mathcal{W}^{c_{x}}(P_t(X_t), P_{t-1}(X_{t-1})) + L_t \; \mathbb{E}_{(x_t, x_{t-1}) \sim \pi_x(X_t, X_{t-1}) }\big[ c_x(x_t, x_{t-1}) \big] \\
    &\le (1+L_t) \mathcal{W}^{c_{x}}(P_t(X_t), P_{t-1}(X_{t-1})) \le (1+L_t)\varepsilon_t
\end{align*}
\end{proof}

For more technical details on working with transportation costs in product spaces, we refer to Appendix A of \citep{gozlan2010transport}.
This implies, that Definition 2 simply describes a gradual shift, where the overall change in the joint  $P_t(X_t, Y_t)$ is dominated by the change in $P_t(X_t)$. This assumption is also common in the domain adaptation literature and relates for instance to the property of being \textit{Probabilistic Transfer Lipschitz} with respect to a labeling function analyzed in \citep{courty2017joint}.

\paragraph{Proving Theorem 1 and Deriving Corollary 1}
Next we conduct the proof of Theorem 1 and discuss the resulting Corollary 1 mentioned in the main paper.
\begin{theorem} 
    Let $\mathcal{L}: \mathcal{Y}\times \mathcal{Y} \rightarrow \mathbb{R}$ be a loss function that is 1-Lipschitz in its second argument and denote the true model performance at time $t$ with $\mathcal{L}_t$.
    If a distribution shift over $\{(X_t, Y_t)\}_{t=1}^T$ is gradually Lipschitz smooth in $X_t$, then:
    \begin{align*}
        | \mathcal{L}_t - \hat{\mathcal{L}}_t^{\textit{IUPM}}| \le \sum_{i=1}^t L_t \; \varepsilon_t
    \end{align*}
\end{theorem}
\begin{proof}
Let $\gamma_t( X_{t-1},X_t)$ be the incremental optimal transport couplings in $X$ and $\Psi_t(X_0|X_t) $ be the composition of all incremental transition probabilities:
\begin{align*}
    \Psi_t(X_0 | X_t) = \int_{\mathcal{X}} \dots \int_{\mathcal{X}} 
    \gamma_1(X_0 | x_1) \gamma_2(x_1 | x_2) \dots \gamma_t(x_{t-1} | X_t) \, 
    dx_1 \dots dx_{t-1}
\end{align*}
Note that in the discrete case, all $\gamma_i(X_{i-1}|X_i)$ are matrices and this composition can equivalently be expressed using iterative matrix multiplication: $\Psi_t(X_0|X_t)=\prod_{i=1}^t \gamma_i(X_{i-1}|X_i)$. For the estimated target distribution $\hat{P}(Y_t|X_t)=\mathbb{E}_{\Psi_t(X_0|X_t)}\left[P(Y_0|X_0) \right]$ it holds:
\begin{align*}
    \lvert \mathcal{L}_t - \hat{\mathcal{L}}_t^{\textit{IUPM}}| &= |\mathbb{E}_{(x_t, y_t) \sim P_t(X_t, Y_t)}\big[ \mathcal{L}(f(x_t), y_t) \big] - \mathbb{E}_{(x_t, \hat{y}_t) \sim P_t(X_t, Y_t)}\big[ \mathcal{L}(f(x_t), \hat{y}_t) \big] \rvert \\ 
    & \le \mathbb{E}_{x_t \sim P_t(X_t)}\big[ \lvert \mathbb{E}_{y_t \sim P_t(Y_t | X_t = x_t)}[ \mathcal{L}(f(x_t), y_t)] - \mathbb{E}_{\hat{y}_t \sim \hat{P}_t(Y_t | X_t = x_t)}[ \mathcal{L}(f(x_t), \hat{y}_t)]  \rvert \big]
\end{align*}
Since we assume $\mathcal{L}$ to be 1-Lipschitz in its second argument $(\mathcal{L}(\cdot, y) \in \text{Lip}_1)$ we know that there exists a cost function $c_y: \mathcal{Y}\times \mathcal{Y} \rightarrow \mathbb{R}^+$ with $\lvert \mathcal{L}(f(x_t), y) -\mathcal{L}(f(f_x), y')\rvert \le c_y(y, y')$ for any fixed $x_t$. Hence, one can apply the Kantorovich-Rubenstein duality (Theorem 5.10 in \citep{villani2009optimal}):
\begin{align*}
    \mathbb{E}_{x_t \sim P_t(X_t)}&\big[ \lvert \mathbb{E}_{y_t \sim P_t(Y_t | X_t = x_t)}[ \mathcal{L}(f(x_t), y_t)] - \mathbb{E}_{\hat{y}_t \sim \hat{P}_t(Y_t | X_t = x_t)}[ \mathcal{L}(f(x_t), \hat{y}_t)]  \rvert \big] \\
    &\le \mathbb{E}_{x_t \sim P_t(X_t)}\big[ \lvert \sup_{\phi \in \text{Lip}_1} \left\{ \mathbb{E}_{y_t \sim P_t(Y_t | X_t = x_t)}[ \phi( y_t)] - \mathbb{E}_{\hat{y}_t \sim \hat{P}_t(Y_t | X_t = x_t)}[ \phi( \hat{y}_t)]\right\}  \rvert \big] \\
    &\le \mathbb{E}_{x_t \sim P_t(X_t)}\Big[ \mathbb{E}_{x_0 \sim \Psi_t(X_0|X_t)}\big[ \sup_{\phi \in \text{Lip}_1} \left\{ \mathbb{E}_{y_t \sim P_t(Y_t | X_t = x_t)}[ \phi( y_t)] - \mathbb{E}_{y_0 \sim P_0(Y_0 | X_0 = x_0)}[ \phi(y_0)]\right\}  \big] \Big] \\
    &= \mathbb{E}_{x_t \sim P_t(X_t)}\Big[ \mathbb{E}_{x_0 \sim \Psi_t(X_0|X_t)}\big[ \mathcal{W}^{c_y}\big(P_t(Y_t | X_t = x_t), P_0(Y_0 | X_0 = x_0)\big)\big]\Big] \\ 
    &\le \mathbb{E}_{x_t \sim P_t(X_t)}\Big[ \mathbb{E}_{x_0 \sim \Psi_t(X_0|X_t)}\big[\sum_{i=1}^t \mathcal{W}^{c_y}\big(P_i(Y_i | X_i = x_i), P_{i-1}(Y_{i-1} | X_{i-1} = x_{i-1})\big)\big]\Big]
\end{align*}
where the last step follows from the fact that the Wasserstein distance satisfies the triangular inequality. Now one can use the assumption that the shift is gradually Lipschitz smooth in $X_t$ for a cost metric $c_x: \mathcal{X} \times \mathcal{X} \rightarrow \mathbb{R}^+$: 
\begin{align*}
\mathbb{E}_{x_t \sim P_t(X_t)}&\Big[ \mathbb{E}_{x_0 \sim \Psi_t(X_0|X_t)}\big[\sum_{i=1}^t \mathcal{W}^{c_y}\big(P_i(Y_i | X_i = x_i), P_{i-1}(Y_{i-1} | X_{i-1} = x_{i-1})\big)\big]\Big] \\
&\le  \mathbb{E}_{x_t \sim P_t(X_t)}\Big[ \mathbb{E}_{x_0 \sim \Psi_t(X_0|X_t)}\big[\sum_{i=1}^t L_t \; c_x( x_i, x_{i-1})\big]\Big]
\end{align*}
Lastly, by decomposing the overall transition probabilities again into the incremental cost-minimizing ones and rearranging all expectations we have:
\begin{align*}
   \mathbb{E}_{x_t \sim P_t(X_t)}\Big[ \mathbb{E}_{x_0 \sim \Psi_t(X_0|X_t)}\big[\sum_{i=1}^t L_t \; c_x( x_i x_{i-1})\big]\Big]
  &=  \big( \sum_{i=1}^t  L_t \;\mathbb{E}_{(x_i, x_{i-1}) \sim \gamma_i(X_i,X_{i-1}) } \big[ c(x_i, x_{i-1})\big]  \big) \\
   &= \sum_{i=1}^t L_t \; \mathcal{W}^{c_x}\big(P_i( X_i), P_{i-1}(X_{i-1})\big)  \le \sum_{i=1}^t L_t \; \varepsilon_t
\end{align*}
\end{proof}
Notice that if an alternative estimator $\tilde{P}(Y_t|X_t)$ is constructed by composing arbitrary incremental couplings $\pi_i$ instead of the cost-optimal ones, the entire derivation is the same until the last part. Thus, bounding the error of the resulting performance estimate $\tilde{\mathcal{L}}_t^{\pi}$ yields:
\begin{align*}
    \lvert \mathcal{L}_t - \tilde{\mathcal{L}}_t^{\pi}| \le  \big( \sum_{i=1}^t L_t \; \mathbb{E}_{(x_i, x_{i-1}) \sim \pi_i(X_i,X_{i-1}) } \big[ c(x_i, x_{i-1})\big]  \big)
\end{align*}
Taking the infimum over all possible couplings $\pi_i$ on the right-hand side to minimize the upper bound results $\pi_i =  \gamma_i$ by the definition of the optimal transport couplings. Hence, choosing incremental couplings based on optimal transport does effectively optimizes an upper bound on the performance estimation error, which implies Corollary 1:
\begin{corollary}
If each $\gamma_t$ is the optimal transport coupling between two consecutive domains $X_{t-1}$ and $X_t$, then this minimizes the estimation error across all possible incremental couplings. 
\end{corollary}

\section{FURTHER DETAILS ON THE IUPM ALGORITHM}
\label{app:B}
In the following, we provide an overview of the IUPM algorithm and subsequently discuss algorithmic complexity.
\subsection{Overview of the IUPM algorithm}
Algorithm 1 describes the proposed procedure to estimate the performance of the model $f$ over time steps $t$.
\begin{algorithm}
\caption{IUPM algorithm}
\begin{algorithmic}[1]
\label{alg:algo1}
\Procedure{IUPM}{$f$, $x_{init}$, $th_U$}
\State \textbf{Input:} $f$: ML model, $x_{init}$: Initialization dataset, $th_U$: Uncertainty threshold
\State \textbf{Output:} $\hat{\mathcal{L}}_t^{\textit{IUPM}}$: Performance estimation, $\mathcal{U}(\hat{\mathcal{L}}_t^{\textit{IUPM}})$: Inherent uncertainty measure

\For{timesteps $t$}
    \State Sample $x_t$ from $\Omega_{t}$
    \State Calculate $\hat{\gamma}_t = \arg \min_{\gamma \in \Gamma} \sum_{x_{t-1} \in \Omega_{t-1}}\sum_{x_{t} \in \Omega_{t}} c(x_{t-1}, x_t) \gamma(x_{t-1},x_t)$ 
    \Comment{Optimal transport matching}
    \State Update $\Psi_t(X_0|X_t)=\prod_{i=1}^t \gamma_i(X_{i-1}|X_i)$
    \Comment{Update transition matrix}
    \State Calculate $\hat{\mathcal{L}}_t^{\textit{IUPM}} = \frac{1}{n_t} \sum_{x_t\in \Omega_t} \mathbb{E}_{\hat{P}(Y_t|X_t=x_t)} \left[ \mathcal{L}(f(X_t), Y_t) \right]$
    \State Get $\mathcal{U}(\hat{\mathcal{L}}_t^{\textit{IUPM}}) = \mathbb{E}_{P(X_t)} \text{SD}_{\hat{P}(Y_t|X_t)}\left[\mathcal{L}(f(X_t), Y_t)\right]$
    
    \If{$\mathcal{U}(\hat{\mathcal{L}}_t^{\textit{IUPM}}) > th_{U}$}
        \State Get $m$ samples $x_t$ based on $\arg\text{top-}m_{x_t \in \Omega_t} \mathcal{U}\left[\mathcal{L}(f(x_t), Y_t)\right]$
        \State Query labels $y_t$ for samples $x_t$ from user
        \State Correct $\Psi_t(X_0|X_t)$ such that $\hat{P}(Y_t|X_t=x_t)$ assigns $y_t$
    \EndIf
\EndFor
\EndProcedure
\end{algorithmic}
\end{algorithm}

\subsection{Discussion of Coupling Estimation and Algorithmic Complexity}
The crucial algorithmic components of IUPM are the computations of the incremental transport couplings $\gamma_i$. They result from solving an optimal transport coupling for which a variety of different computational approaches exist \citep{peyre2019computational}. We utilized a popular approach to increase the estimation efficiency using entropic regularization \citep{cuturi2013sinkhorn, flamary2021pot}. Let $KL(\cdot |\cdot)$ denote the KL-divergence between two distributions and let $\lambda$ be a hyperparameter capturing the regularization strength, then the objective reads:
\begin{align*}
\min_{\gamma \in \Gamma} \sum_{x_0 \in \Omega_0 }\sum_{x_1 \in \Omega_1 } c(x_0, x_1) \gamma(x_0,x_1) + \lambda KL(\gamma |\hat{p}_0 \otimes \hat{p}_1 ) \quad \text{with}  \quad \Gamma =\{\gamma \in \mathbb{R}^{n_0 \times n_1}\; |\; \gamma \mathbf{1}_{n_1} = \hat{p}_0,  \gamma^T\mathbf{1}_{n_0} = \hat{p}_1\}
\end{align*}
This can efficiently be solved using the Sinkhorn algorithm \citep{cuturi2013sinkhorn}, which has sample complexity of $\mathcal{O}(1/\sqrt{n})$ and time complexity of $\mathcal{O}(n^2)$, where $n$ is the number of samples to be matched from each domain \citep{genevay2019sample}. Note that IUPM requires solving one optimal transport problem for every time point of assessment during model deployment.  

\section{ADDITIONAL DETAILS ON EXPERIMENTS}
\label{app:C}
In this section, we provide additional details on the experiments performed.

\subsection{Computational Environment}
All numerical experiments are implemented in Python (version 3.9.13) using PyTorch (version 1.13.0) and have been computed on an Nvidia RTX A5000 GPU with CUDA 11.7 and two physical AMD EPYC 7502P 32-Core CPUs running on Linux Ubuntu.

\subsection{Method Implementation}
\paragraph{Baselines}
For the four confidence-based baseline methods, we rely on the implementations provided by \citep{gargleveraging}. For the ATC method \citep{gargleveraging}, we use the author's proposed maximum confidence score function.
\paragraph{IUPM and NIPM}
For our IUPM and NIPM implementation, we rely on the entropic regularization optimal transport implementation with logarithmic Sinkhorn by \citep{flamary2021pot}. As the cost function for calculating optimal transport, we use the squared Euclidean distance throughout all experiments. Similar to the baseline methods, we also consider the model accuracy as the loss criterion $\mathcal{L}$ to evaluate performance in the conducted experiments.
In all experiments the uncertainty intervention threshold $\mathcal{U}(\mathcal{L}_t) > 0.1$ is set to trigger a relabeling of 50\% of the samples in step $t$
\subsection{Experiments}
\paragraph{Translation and Rotation in Input Space}
As for the synthetic two-dimensional datasets, we used the data generator functionality provided by \citep{scikit-learn}. The "Clusters" data set is generated using the make\_blobs function with a distance parameter of $1.0$. The "Moons" and "Circles" data sets are generated using the corresponding functions with a noise parameter of $0.2$ and a circle factor of $0.3$.
For the training and initialization step, a training set of $800$ samples is generated, from which a validation and initialization set $\Omega_0$ of $200$ samples is partitioned. In each consecutive step, a set $\Omega_k$ with a different random seed is generated. We then apply a shift to the set corresponding to the step $k$, i.e., $k\cdot 2^{\circ}$ for rotation and $k\cdot 0.02$ for translation.\\
For the synthetic data, we use a Random Forest Classifier (RF) and an XGBoost Classifier (XGB) \citep{chen2016xgboost} with $50$ estimators and a maximum depth of $5$ as well as a Multilayer Perceptron (MLP) with a single hidden layer of size $128$. We use a regularization parameter of $10^{-4}$ for optimal transport matching.
\paragraph{Monitoring Performance Degradation due to
Image Perturbations}
For the experiment based on MNIST, we apply three different affine transformations to the original digits. For this, we adapt the corresponding image perturbation implementation introduced in \citep{mu2019mnist} to the continuous setting. The used LeNet model has been trained for $100$ epochs with early stopping based on patience of $10$ epochs and PyTorch’s Adam optimizer with a batch size of $16$ and a learning rate of $1e-3$. For optimal transport matching, we use the representations after the second fully connected layer of the LeNet model and a regularization parameter of $1$. The sets $\Omega_t$, including the initialization, set $\Omega_0$, each consists of $200$ distinct samples from the test set. For ImageNet we fine-tuned a pre-trained ResNet-50 \citep{he2016deep} model on the ten classes included in the Imagenette subset \citep{Howard_Imagenette_2019} of ImageNet. For this we use PyTorch’s Adam optimizer with a batch size of $16$ and a learning rate of $1e-5$. We fine-tuned the model for $50$ epochs with early stopping and a patience of $10$ epochs. We analyzed 7 different shifts from $\texttt{ImageNet-c}$ \citep{hendrycks2018benchmarking} that can be considered gradual, each based on $500$ samples. Note that $\texttt{ImageNet-c}$ provides each shift in five predefined strengths and we have additionally interpolated all shifts for a total number of $20$ steps. All incremental couplings have been computed based on the network activation of the last layer before the classification layer with a Sinkhorn regularization parameter of $1e-4$.
\paragraph{Monitoring Performance Degradation due to
Real-World Temporal Shifts}
For the pre-processing steps as well as the network architecture used for the portrait experiment, we rely on the implementation provided by \citep{yao2022wild}. By this, images are of shape $32\times32$ and the used YearbookNetwork consists of four convolutional blocks with 32 channels and a single linear classification layer. The model has been trained for $300$ epochs with early stopping based on patience of $5$ epochs and PyTorch’s Adam optimizer with a batch size of $32$ and a learning rate of $1e-3$. For the optimal transport matching, we use the representations after the last convolutional block of the YearbookNetwork and a regularization parameter of $1e-3$.

\section{ADDITIONAL EXPERIMENTAL RESULTS}
\label{app:D}
In this section, we present additional results that extend our evaluation.
\paragraph{Confidence Intervals} To validate the statistical significance of our results on the synthetic and MNIST datasets in Table \ref{tab:syn_shifts} and Table \ref{tab:mnist_shifts}, we provide the confidence intervals in Tables \ref{tab:syn_shift_conf} and \ref{tab:mnist_shift_conf}, respectively.
Both tables show that there is very little change across five different random seeds.
\begin{table*}
\caption{Mean Average Error (MAE) for five different random seeds between ground truth and estimated accuracy using baseline methods and IUPM for the three synthetic data sets and three different models.}
\centering
\resizebox{\textwidth}{!}{
\begin{tabular}{@{}lccccccccc@{}}
\toprule
\multirow{2}{*}{\textbf{Method}} &
\multicolumn{3}{c}{\textbf{Clusters}} &
\multicolumn{3}{c}{\textbf{Moons}} &
\multicolumn{3}{c}{\textbf{Circles}} \\ \cmidrule(l){2-10} 
& \textbf{RF} & \textbf{XGB} & \textbf{MLP} & \textbf{RF} & \textbf{XGB} & \textbf{MLP} & \textbf{RF} & \textbf{XGB} & \textbf{MLP} \\ \midrule
ATC & 0.4413$_{\pm 0.0026}$ & 0.4664$_{\pm 0.0040}$ & 0.4348$_{\pm 0.0015}$ & 0.3788$_{\pm 0.0028}$ & 0.3594$_{\pm 0.0015}$ & 0.3679$_{\pm 0.0014}$ & 0.3514$_{\pm 0.0012}$ & 0.3442$_{\pm 0.0016}$ & 0.3531$_{\pm 0.0014}$ \\
AC & 0.4313$_{\pm 0.0028}$ & 0.4875$_{\pm 0.0030}$ & 0.4402$_{\pm 0.0017}$ & 0.3243$_{\pm 0.0020}$ & 0.3741$_{\pm 0.0016}$ & 0.3529$_{\pm 0.0014}$ & 0.2760$_{\pm 0.0012}$ & 0.3325$_{\pm 0.0013}$ & 0.3240$_{\pm 0.0008}$ \\
DOC & 0.4360$_{\pm 0.0027}$ & 0.4527$_{\pm 0.0031}$ & 0.4488$_{\pm 0.0017}$ & 0.3632$_{\pm 0.0021}$ & 0.3493$_{\pm 0.0017}$ & 0.3701$_{\pm 0.0014}$ & 0.3574$_{\pm 0.0012}$ & 0.3418$_{\pm 0.0013}$ & 0.3473$_{\pm 0.0009}$ \\
IM & 0.4525$_{\pm 0.0024}$ & 0.4662$_{\pm 0.0030}$ & 0.4705$_{\pm 0.0015}$ & 0.3955$_{\pm 0.0022}$ & 0.3599$_{\pm 0.0017}$ & 0.3800$_{\pm 0.0015}$ & 0.3559$_{\pm 0.0012}$ & 0.3431$_{\pm 0.0013}$ & 0.3488$_{\pm 0.0011}$ \\
NIPM & 0.4225$_{\pm 0.0034}$ & 0.4186$_{\pm 0.0037}$ & 0.4673$_{\pm 0.0014}$ & 0.2482$_{\pm 0.0032}$ & 0.2286$_{\pm 0.0033}$ & 0.2319$_{\pm 0.0024}$ & 0.0775$_{\pm 0.0012}$ & 0.0791$_{\pm 0.0022}$ & 0.0742$_{\pm 0.0016}$ \\
IUPM & 0.2914$_{\pm 0.0132}$ & 0.2894$_{\pm 0.0139}$ & 0.3035$_{\pm 0.0120}$ & 0.0781$_{\pm 0.0091}$ & 0.0793$_{\pm 0.0085}$ & 0.1020$_{\pm 0.0075}$ & 0.0352$_{\pm 0.0044}$ & 0.0390$_{\pm 0.0040}$ & 0.0359$_{\pm 0.0040}$ \\
IUPM$_{UI}$ & 0.0322$_{\pm 0.0033}$ & 0.0307$_{\pm 0.0037}$ & 0.0331$_{\pm 0.0029}$ & 0.0250$_{\pm 0.0019}$ & 0.0250$_{\pm 0.0018}$ & 0.0230$_{\pm 0.0015}$ & 0.0136$_{\pm 0.0017}$ & 0.0144$_{\pm 0.0021}$ & 0.0138$_{\pm 0.0020}$ \\
\hline
\end{tabular}}
\label{tab:syn_shift_conf}
\end{table*}
\begin{table*}
\caption{Mean Average Errors (MAE) for five different random seeds between ground truth and estimated accuracy for a LeNet across three different shifts on the MNIST data set.}
\centering
\footnotesize
\begin{tabular}{lccc}
\toprule
Method & Rotation & Scaling & Translation \\
\midrule
ATC & 0.4836$_{\pm 0.0082}$ & 0.1763$_{\pm 0.0086}$ & 0.3111$_{\pm 0.0024}$ \\
AC & 0.4931$_{\pm 0.0086}$ & 0.2292$_{\pm 0.0059}$ & 0.3842$_{\pm 0.0061}$ \\
DOC & 0.5181$_{\pm 0.0086}$ & 0.2538$_{\pm 0.0060}$ & 0.4090$_{\pm 0.0061}$ \\
IM & 0.6282$_{\pm 0.0083}$ & 0.6166$_{\pm 0.0061}$ & 0.5686$_{\pm 0.0097}$ \\
NIPM & 0.2187$_{\pm 0.0099}$ & 0.0676$_{\pm 0.0112}$ & 0.3110$_{\pm 0.0138}$ \\
IUPM & 0.0985$_{\pm 0.0064}$ & 0.0442$_{\pm 0.0077}$ & 0.1263$_{\pm 0.0084}$ \\
IUPM$_{UI}$ & 0.0719$_{\pm 0.0093}$ & 0.0438$_{\pm 0.0052}$ & 0.0777$_{\pm 0.0064}$ \\
\bottomrule
\end{tabular}
\label{tab:mnist_shift_conf}
\end{table*}
\paragraph{Labeling Interventions} Reducing the need for human interactions for our proposed IUPM method with human intervention limits the manual effort for a reliable performance estimate. In the main paper, we have already shown in Table \ref{tab:sampling_methods} that the number of interventions for our proposed Uncertainty Intervention (UI) method is significantly lower than for the other methods, while the quality of the estimate is better or on par. We performed an ablation study in Table \ref{tab:intervention_steps_abl} to validate that the reverse conclusion that UI provides a better performance estimation for the same number of interventions also holds true. In this experiment, we trigger the interventions at the same steps, taking as a reference the exceeding of the threshold given by the UI method.
\begin{table*}[h]
\caption{Comparison of the three intervention methods when provided the same annotation budget for the synthetic datasets}
\centering
\begin{tabular}{l*{9}{c}}
\toprule
\multirow{2}{*}{Method} & \multicolumn{3}{c}{Clusters} & \multicolumn{3}{c}{Moons} & \multicolumn{3}{c}{Circles} \\
\cmidrule(lr){2-4} \cmidrule(lr){5-7} \cmidrule(lr){8-10}
& RF & XGB & MLP & RF & XGB & MLP & RF & XGB & MLP \\
\midrule
IUPM$_{RI}$ & 0.0506 & 0.0485 & 0.0531 & 0.0336 & 0.0319 & 0.0341 & 0.0202 & 0.0165 & 0.0244 \\
IUPM$_{CEI}$ & 0.0600 & 0.0656 & 0.0575 & 0.0406 & 0.0274 & 0.0233 & 0.0230 & 0.0280 & 0.0292 \\
IUPM$_{UI}$ & \textbf{0.0270} & \textbf{0.0272} & \textbf{0.0265} & \textbf{0.0244} & \textbf{0.0242} & \textbf{0.0222} & \textbf{0.0160} & \textbf{0.0157} & \textbf{0.0158} \\
\bottomrule
\end{tabular}
\label{tab:intervention_steps_abl}
\end{table*}
\paragraph{Uncertainty Threshold} In the following we discuss the effect of the uncertainty threshold introduced by our method to trigger the intervention steps. Table \ref{tab:threshold} underlines the intuitive effect that lowering the threshold will lead to a higher number of interventions and thus queried ground truth labels over time. This additional label information helps to correct the estimation, leading to a trade-off between estimation quality and human intervention. In a practical application, this trade-off may be determined by external factors, such as a limited number of possible human interventions, making this hyperparameter useful for tailoring a monitoring system to a specific use case and application.

In our case, we determined the threshold by assessing the number of cumulative interventions combined with the average gain in performance estimation quality per intervention between two steps $\frac{\Delta MAE}{\Delta n_I}$. The value added per intervention begins to saturate between 0.08 and 0.12, so that the benefit of human intervention diminishes. In addition, we believe that an average of about 13 interventions in 100 steps is still reasonable for the observed improvement in the estimate. Based on these results from the synthetic experiments, we chose an uncertainty threshold of $0.1$, which proved to be robust in all other experiments.
\begin{table*}
\centering
\caption{Evaluation of the relative benefit on performance per intervention step across intervention thresholds from $0.20$ to $0.02$. This ratio is calculated on the reported average intervention number and average MAE across all three synthetic datasets and evaluated models.}
\vspace{0.5em}
\footnotesize
\setlength{\tabcolsep}{4pt}
\resizebox{\textwidth}{!}{
\begin{tabular}{lcccccccccc}
\toprule
Method & 0.20 & 0.18 & 0.16 & 0.14 & 0.12 & 0.10 & 0.08 & 0.06 & 0.04 & 0.02 \\
\midrule
MAE & 0.0356 & 0.0350 & 0.0312 & 0.0264 & 0.0231 & 0.0221 & 0.0213 & 0.0211 & 0.0201 & 0.0222 \\
\addlinespace
$n_I$ & 5.33 & 6.67 & 7.67 & 8.67 & 11.00 & 12.67 & 16.67 & 23.00 & 33.00 & 57.00 \\
\addlinespace
$\frac{\Delta MAE}{\Delta n_I}$ & & 4.23e-04 & 3.79e-03 & 4.82e-03 & 1.43e-03 & 5.74e-04 & 2.03e-04 & 3.62e-05 & 9.59e-05 & -8.71e-05 \\
\bottomrule
\end{tabular}}
\label{tab:threshold}
\end{table*}
\paragraph{Further Results on ImageNet-c}
Figures \ref{fig:imagentshifts1} and \ref{fig:imagentshifts2} provide further insight into the ImageNet-c experiments in Table \ref{tab:imagenet} by presenting the performance estimates and intervention steps over time.
\begin{figure}
    \centering
    \includegraphics[width=0.9\linewidth]{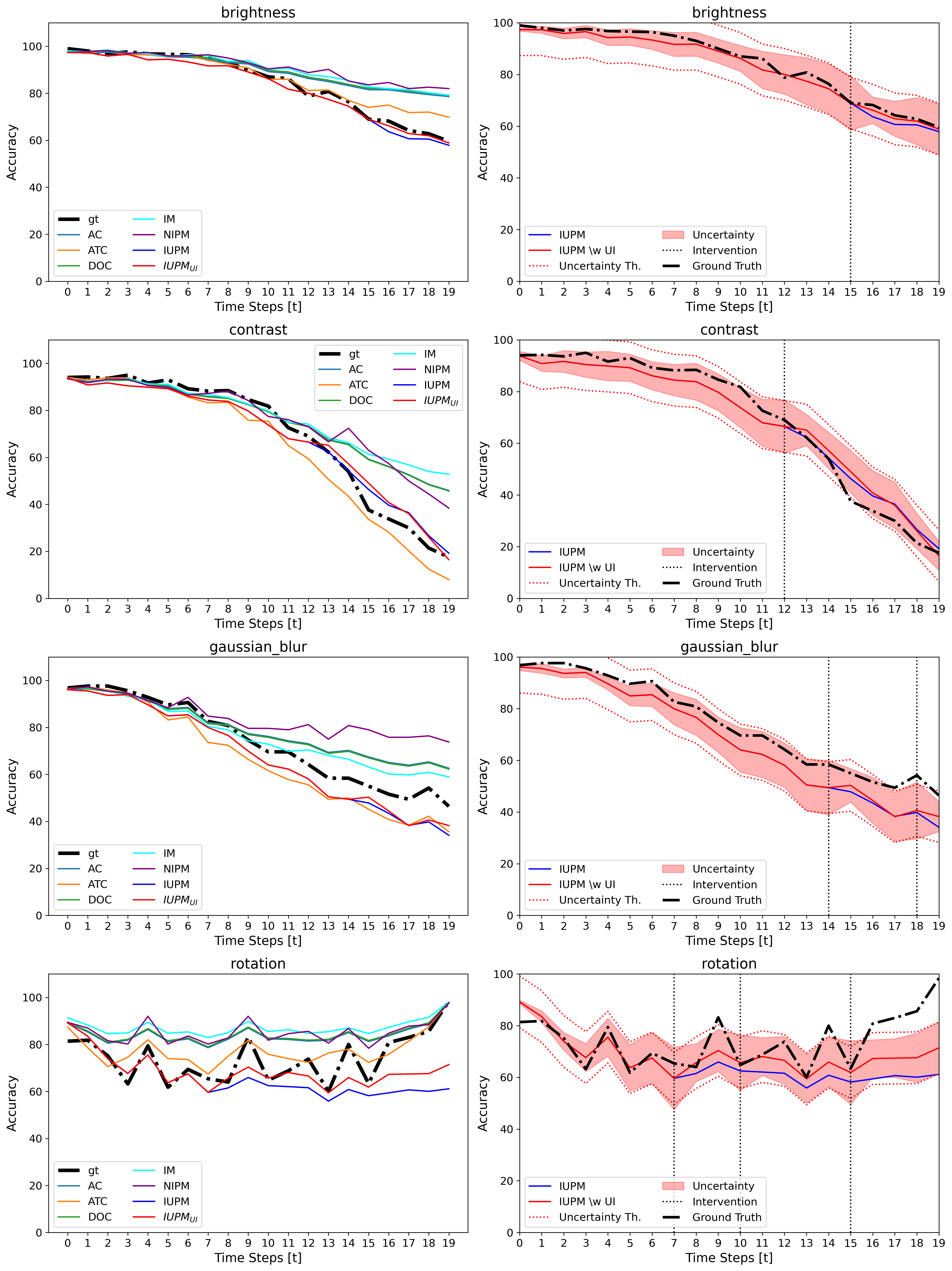}
    \caption{Results on different gradual shifts from ImageNet-c on a ResNet-50 model with 500 samples. Left: Evolution of actual ground truth performance (gt) and different performance estimation methods over time, where at each time point the shift strength increases. Right: Corresponding visualization of IUPM's inherent uncertainty measure and effect of triggered labeling interventions.}
    \label{fig:imagentshifts1}
\end{figure}

\begin{figure}
    \centering
    \includegraphics[width=0.9\linewidth]{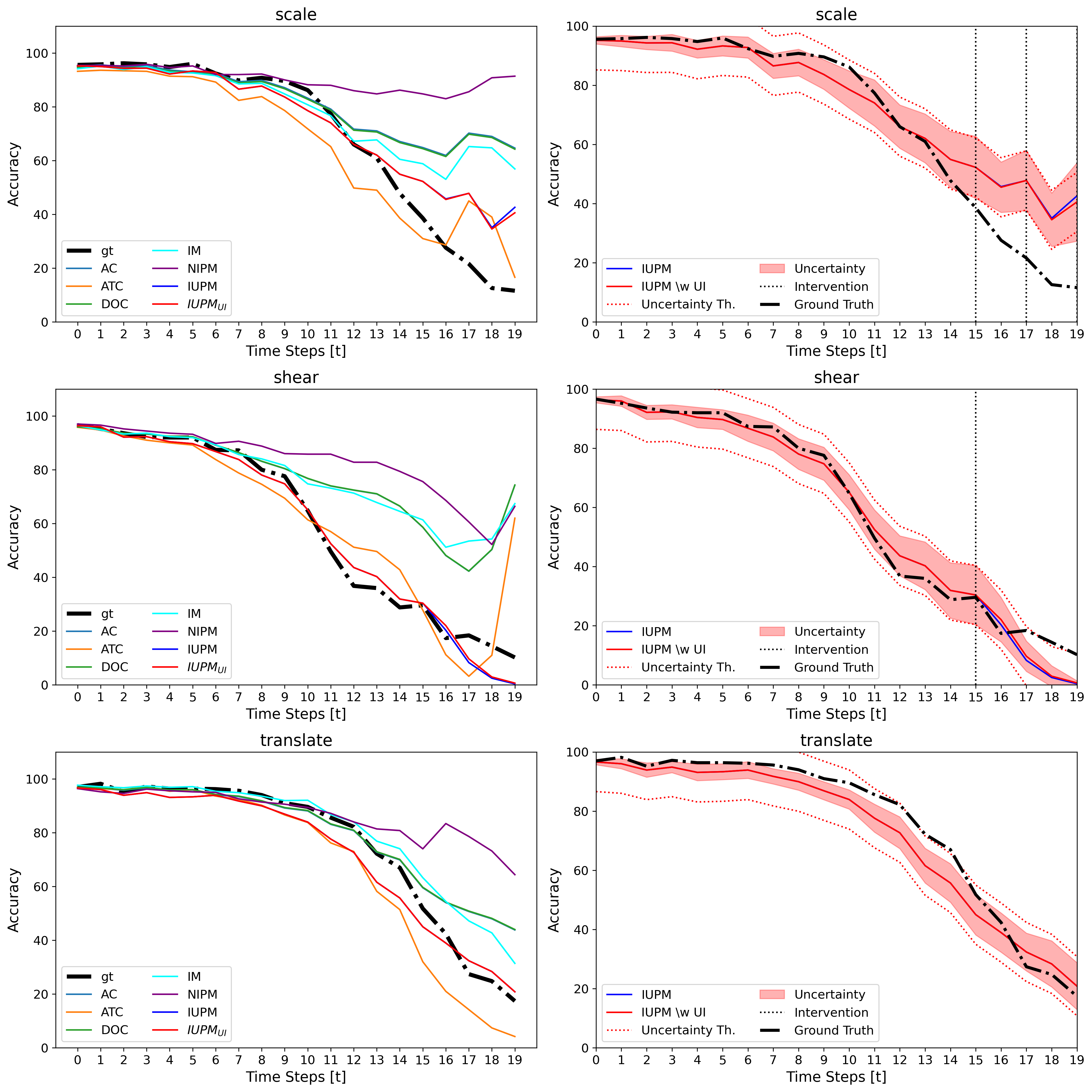}
    \caption{Results on different gradual shifts from ImageNet-c on a ResNet-50 model with 500 samples (continued). Left: Evolution of actual ground truth performance (gt) and different performance estimation methods over time, where at each time point the shift strength increases. Right: Corresponding visualization of IUPM's inherent uncertainty measure and effect of triggered labeling interventions.}
    \label{fig:imagentshifts2}
\end{figure}

\end{document}